%% file: main.tex
\documentclass{article}

\usepackage{microtype}
\usepackage{graphicx}
\usepackage{subcaption}
\usepackage{booktabs} 

\usepackage{hyperref}

\usepackage{nccmath}

\usepackage[preprint]{icml2026}

\usepackage{amsmath}
\usepackage{amssymb}
\usepackage{mathtools}
\usepackage{amsthm}

\usepackage[capitalize,noabbrev]{cleveref}

\theoremstyle{plain}
\newtheorem{theorem}{Theorem}[section]
\newtheorem{proposition}[theorem]{Proposition}
\newtheorem{lemma}[theorem]{Lemma}

\theoremstyle{definition}

\theoremstyle{remark}

\icmltitlerunning{go-\mHC: Direct Parameterization of \mHC\ via Generalized Orthostochastic Matrices}

\newcommand{\R}{\mathbb{R}}

\newcommand{\bB}{\mathsf{B}}
\newcommand{\cO}{\mathcal{O}}
\newcommand{\cU}{\mathcal{U}}
\newcommand{\cP}{\mathcal{P}}
\newcommand{\cQ}{\mathcal{Q}}
\newcommand{\cH}{\mathcal{H}}
\newcommand{\cF}{\mathcal{F}}
\newcommand{\vx}{\mathbf{x}}
\newcommand{\Spec}{\mathsf{Spec}}

\newcommand{\mHC}{\textit{m}HC}

\newif\ifshowcontent
\showcontenttrue  

\usepackage{enumitem}

\makeatletter
\def\mathcolor#1#{\@mathcolor{#1}}
\def\@mathcolor#1#2#3{%
  \protect\leavevmode
  \begingroup
    \color#1{#2}#3%
  \endgroup
}
\makeatother

\begin{document}

\twocolumn[
  \icmltitle{go-\mHC: Direct Parameterization of Manifold-Constrained\\Hyper-Connections via Generalized Orthostochastic Matrices}

  \icmlsetsymbol{equal}{*}

  \begin{icmlauthorlist}
    \icmlauthor{Torque Dandachi}{comp}
    \icmlauthor{Sophia Diggs-Galligan}{comp}
  \end{icmlauthorlist}

  \icmlaffiliation{comp}{Independent, San Francisco, United States of America}

  \icmlcorrespondingauthor{Torque Dandachi}{hi@itstorque.com}
  \icmlkeywords{}

  \vskip 0.3in
]
\printAffiliationsAndNotice{}

\begin{abstract}
    Doubly stochastic matrices enable learned mixing across residual streams, but parameterizing the set of doubly stochastic matrices (the Birkhoff polytope) exactly and efficiently remains an open challenge. Existing exact methods scale factorially with the number of streams ($d$), while Kronecker-factorized approaches are efficient but expressivity-limited. We introduce a novel exact parameterization grounded in the theory of \underline{g}eneralized \underline{o}rthostochastic matrices, which scales as $\mathcal{O}(d^3)$ and exposes a single hyperparameter $s$ which continuously interpolates between a computationally efficient boundary and the fully expressive Birkhoff polytope. Building on Manifold-Constrained Hyper-Connections (\mHC), a framework for learned dynamic layer connectivity, we instantiate this parameterization in go-\mHC. Our method composes naturally with Kronecker-factorized methods, substantially recovering expressivity at similar FLOP costs. Spectral analysis indicates that go-\mHC\ fills the Birkhoff polytope far more completely than Kronecker-factorized baselines. On synthetic stream-mixing tasks, go-\mHC\ achieves the minimum theoretical loss while converging up to $10\times$ faster. We validate our approach in a 30M parameter GPT-style language model. The expressivity, efficiency, and exactness of go-\mHC\ offer a practical avenue for scaling $d$ as a new dimension of model capacity.
\end{abstract}

\section{Introduction}

The performance and stability of modern machine learning models increasingly depends on the geometric structure of their underlying matrix parameterizations \cite{bronstein2021geometricdeeplearninggrids, arjovsky2016unitaryevolutionrecurrentneural, saxe2014exactsolutionsnonlineardynamics, zhao2025symmetryneuralnetworkparameter}. While traditional constraints—such as symmetry or orthogonality—have long served as foundational tools \cite{arjovsky2016unitaryevolutionrecurrentneural, wisdom2016fullcapacityunitaryrecurrentneural, henaff2017recurrentorthogonalnetworkslongmemory}, recent breakthroughs have pivoted toward more complex geometric structures \cite{Papillon_2025, bronstein2021geometricdeeplearninggrids}.
A promising frontier for addressing these limitations lies in leveraging manifold theory to map unconstrained optimization spaces onto computationally tractable, constrained geometries \cite{xie2025mhc, chen2026discoveringsymmetrygroupsflow, smith2014optimizationtechniquesriemannianmanifolds, lu2021physicsinformedneuralnetworkshard}. Of particular interest is the Birkhoff polytope \cite{adams2011rankingsinkhornpropagation, tay_sparse_sinkhorn_2020, sinkformers_2022, xie2025mhc}, the manifold of doubly stochastic matrices. By enforcing unit row and column sum constraints, the Birkhoff polytope offers a rigorous framework for learning permutations, alignments, and hyper-connections, yet its practical utility has been bottlenecked by the trade-off between exactness and computational scalability \cite{xie2025mhc}.

\citet{sinkformers_2022} demonstrated that constraining the attention pattern to doubly stochastic matrices both improved accuracy in vision and language models while providing an elegant interpretation of learned parameters as probability trajectories \cite{sinkformers_2022, shahbazi2025espformerdoublystochasticattentionexpected}. \citet{tay_sparse_sinkhorn_2020} investigated memory-efficient implementations of Sinkhorn Attention due to it's ability to mix internal representations and have an associated smooth optimization landscape.

A more recent application of such manifold constraints is in the generalization of residual connections \cite{sinkformers_2022, xie2025mhc}.
Residual connections have become a crucial component in stabilizing the training of deep neural networks \cite{he2016deep, he2016identitymappingsdeepresidual} and a critical building block of modern large language models \cite{gpt2, att_is_all_you_need}.
Ongoing work investigates dynamic generalizations of the residual streams aimed at increasing the topological complexity of residual function connectivity by increasing the number of streams, $d$, while preserving the identity mapping property of the residual stream \cite{zhu2025hyper}. 

\begin{figure*}[t] 
    \centering
    \includegraphics[width=\textwidth]{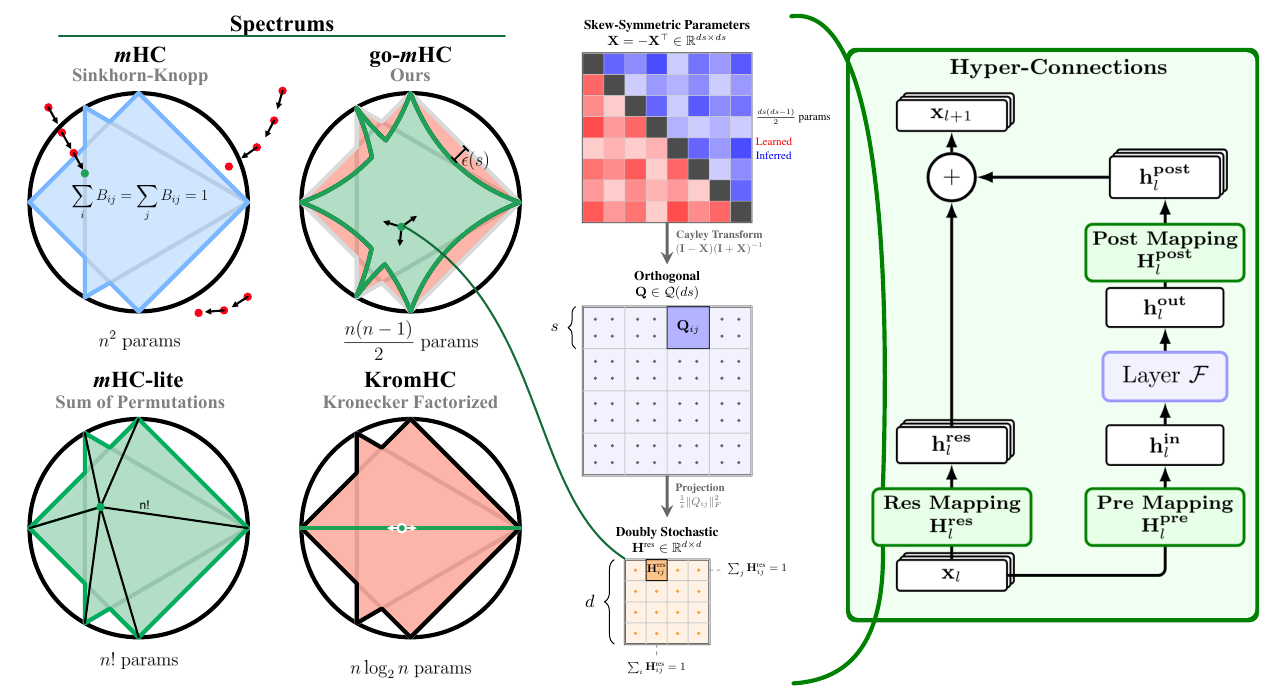}
    \caption{\textbf{Spectral analysis and architectural integration of go-\mHC.} 
\textit{(Left) Spectral Comparison:} Illustration of spectral reach and approximation gaps across several manifold parameterization methods.
\textbf{\mHC} (and the \textit{Sinkhorn-Knopp} algorithm) is not guaranteed to always remain within the manifold after finite iterations iterations (red points). 
\textbf{go-\mHC\ (Ours)} closely approximates the Birkhoff polytope $\mathsf{B}_d$ within an $\epsilon(s)$ error margin using only $\frac{d(d-1)}{2}$ parameters. 
\textbf{\mHC-lite} provides full coverage but requires $d!$ parameters, while \textbf{KromHC} is restricted to tensors of $\mathsf{B}_2$, disallowing non-adjacent cycles of length $>2$ and severely restricting the space of learnable matrices.
\textit{(Middle) Parameterization Pipeline:} Our method maps learned skew-symmetric parameters $\mathbf{X} = -\mathbf{X}^\top \in \mathbb{R}^{ds \times ds}$ to an orthogonal matrix $\mathbf{Q} \in \mathcal{Q}(ds)$ via the Cayley Transform $(\mathbf{I}-\mathbf{X})(\mathbf{I}+\mathbf{X})^{-1}$. The final doubly stochastic matrix $\mathbf{B} \in \mathbb{R}^{d \times d}$ is obtained by a block-wise Frobenius norm projection $\frac{1}{s} \|\mathbf{Q}_{ij}\|_F^2$.
\textit{(Right) Hyper-Connections:} The transformation $\cH^{\text{res}}$ is integrated into the residual stream $\mathbf{x}_{l+1} = \cH_l^{\text{res}}\mathbf{x}_l + \text{Layer } \mathcal{F}(\mathbf{x}_l)$, where the doubly stochastic constraint ensures training stability and improved gradient flow.}
    \label{fig:one}
\end{figure*}

Manifold-Constrained Hyper-Connections (\textbf{\mHC}) \cite{xie2025mhc} provides evidence that constraining the residual connections to the interior of the Birkhoff polytope significantly improves training stability and scalability. These dynamic generalizations reveal the possibility for a new scaling dimension complementing that of traditional scaling laws: the number of residual streams, $d$. Follow-up works have already implemented \mHC\ as the default backbone due to its superior modeling capabilities \cite{cheng2026conditionalmemoryscalablelookup}.

In order to further explore the feasibility of this scaling dimension, two works have since introduced the use of exact parameterizations to prevent error accumulation across layers \cite{yang2026mhclite, zhou2026kromhc}. These methods additionally avoid the overhead introduced by iterative normalization and overcome the need for custom infrastructure optimization. 
\textbf{\mHC-lite} \cite{yang2026mhclite} invokes the Birkhoff--von Neumann theorem to parameterize $\cH^{\text{res}}$ as a convex combination of all $d!$ permutation matrices.
However, the $\cO(d \cdot d!)$ parameter count explodes super-exponentially: at $d=8$, $8!=40{,}320$ permutations are required, rendering this approach memory-prohibitive without sampling approximations that restore an approximation gap.
Conversely, \textbf{KromHC} \cite{zhou2026kromhc} uses the Kronecker products of small doubly stochastic matrices to achieve exact constraints at $\cO(d^2)$ complexity, but in doing so, severely trades off expresivity.

We propose a complementary solution rooted in the algebraic theory of generalized orthostochastic matrices \cite{Gutkin_2013, nechita2023generalized, Korzekwa_2018, Shahbeigi_2021}.
Our framework for exactly learning doubly stochastic matrices attempts to balance the need for efficiency and expressivity, while exposing a hyperparameter $s$ that allows for tunability across these attributes. This opens up the possibility for further exploration and new approaches towards scaling Hyper-Connections.

In this work, we focus on \mHC\ as an example application, but our \underline{g}eneralized \underline{o}rthostochastic construction (go-\mHC) can be applied more broadly as a way to parameterize and learn the interior of the Birkhoff polytope. We characterize our parameterization in the context of a synthetic data task mixing $d$ residual streams and a 30M parameter GPT-style language model.

We specifically contribute analytics and explorations that indicate that our method:
\begin{enumerate}
\item \textbf{Achieves faster convergence and better reachability} than parametrizations invoking the Birkhoff von-neumann theorem, as in \mHC-lite and KromHC

\item \textbf{Scales as $\cO(d^3)$} with the number of residual streams $d$, both in FLOP complexity and parameter count, surpassing the super-exponential scaling of \mHC-lite and KromHC's component factors.

\item \textbf{Requires no custom CUDA kernels}, relying only on batched linear solves and tensor contractions.

\item \textbf{Achieves exact double stochasticity by construction}, with no iterations and no approximation gap.

\item \textbf{Includes a hyperparameter $s$, which serves as a tunable degree of freedom}, interpolating between the orthostochastic boundary ($s=1$) and the dense interior of the Birkhoff polytope ($s \to \infty$).

\end{enumerate}

We believe that this method presents a potential avenue for scaling Hyper-Connections with tractable complexity while maintaining significant expressivity. Additionally, our method can be used simultaneously with KromHC to allow for larger Kronecker factors due to the mitigation of the super-exponential scaling introduced in \mHC-lite. 

\section{Related Work}

\paragraph{Residual connections and their generalizations.}
He et al.\ \cite{he2016deep} introduced skip connections that enable gradient flow via identity shortcuts in ResNets.
This paradigm has evolved to become a key component in the design of large language models (LLMs) owing to its stability and scalability \cite{att_is_all_you_need, gpt2}.
DenseNet \cite{huang2017densely}, FractalNet \cite{larsson2017fractalnet}, and subsequent works explored richer connectivity patterns passing intermediate features across networks. Other works have explored novel ways of increasing the effective size of the residual stream \cite{mak2025residualmatrixtransformersscaling, pagliardini2024denseformer, ResiDual_DoublePaper, AltUp}, examples include: ResiDual \cite{ResiDual_DoublePaper}, which explored improvements to layer normalization by doubling the size of the residual stream, and AltUp \cite{AltUp}, which widened the effective hidden dimension while maintaining low compute by processing partitioned representations in different layers.

Hyper-Connections (HC) \cite{zhu2025hyper} proposed the first learnable, multi-stream generalization of residual streams allowing for dynamic mixing matrices. A sketch of one such residual block is shown in the green box of figure~\ref{fig:one}. HC can theoretically enable models to learn depth and width connections to dynamically rearrange layer connectivity and potentially adds a new avenue for scaling.
Multiple approaches have expanded on this work since, such as splitting
the feature dimension into independent sub-streams \cite{zhu2025frac} and applying HC to vision and graph neural networks \cite{zhu2025hyper, mhcgnn}.

However, the unconstrained dynamic connectivity of residual blocks in HC removes the explicit guarantee of the identity property -- introducing the risk of gradient instability after repeated applications of the mixing terms \cite{xie2025mhc}. This instability impacts deep networks and restricts scalability due to vanishing and exploding gradients. Subsequent works investigate constraining this connectivity to a manifold that bounds gradient updates through the spectral norm in order to address this instability.

\paragraph{Manifold-Constrained Hyper-Connections.}
Building on top of Hyper-Connections, Manifold-Constrained Hyper Connections (\mHC) constrain the dynamic residual matrices to a manifold to improve training stability while still allowing for dynamic connections \cite{xie2025mhc}. \mHC\ primarily changes the definition of $\cH_l^{\text{res}}$ such that it converges to the Birkhoff polytope $\mathsf{B}_d$, where $\mathsf{B}_d$ denotes the closed set of doubly stochastic matrices, defined as positive matrices with the following property:
$$
\mathbf{M}^\top \mathbf{J}_d = \mathbf{M} \mathbf{J}_d = \mathbf{J}_d,
$$
where $\mathbf{J}_d$ denotes the $d\times d$ ones matrix. See the blue spectral region in figure~\ref{fig:one} of \mHC.

While \mHC\ attempts to enforce this constraint after weight updates using the Sinkhorn-Knopp (SK) algorithm \cite{sinkhorn1967}, an iterative normalization method for generating doubly stochastic matrices, recent works \cite{yang2026mhclite, zhou2026kromhc} aim to build HC networks where the learned map is exactly constrained to the polytope. This has two major implications: 1. implementing the SK kernel requires careful kernel-level optimization for efficiency, and 2. finite SK iterations introduce an approximation gap that accumulates with network depth.

\mHC-lite \cite{yang2026mhclite} attempts to circumvent the approximation gap and customized kernel by making use of the Birkhoff--von Neumann theorem to define the learned connections as a convex combination of the $d!$ permutation matrices (where $d$ is the number of residual streams), at an $\cO(dC\cdot d!)$ cost. This method works well for small $d$, however, this is not tenable in the limit of scaling $d > 8$. KromHC \cite{zhou2026kromhc} sacrifices expressivity and complete coverage of $\mathsf{B}_d$ by using tensor products to train residual connections on a much smaller set of exactly double-stochastic matrices with the added benefit of costing only $\cO(d^2 C)$.

\paragraph{Cayley transforms in neural networks.}

The Cayley transform has been used to parameterize orthogonal matrices in recurrent networks \cite{helfrich2018cayley}, and comparisons with the matrix exponential show similar quality and better stability at lower FLOP cost \cite{lezcano2019cheap}.
Li et al. (\citet{li2020efficient}) further demonstrate that iterative Cayley-based optimization outperforms matrix exponential approaches for Stiefel manifold problems.

\paragraph{Doubly Stochastic \& Orthostochastic matrices.}

Multiple works have investigated the use of doubly stochastic matrices in the context of neural network parameterization \cite{sinkformers_2022, adams2011rankingsinkhornpropagation, tay_sparse_sinkhorn_2020}. \citet{Gutkin_2013} introduced generalized orthostochastic matrices, and more recently, \citet{nechita2023generalized} proved that the set of generalized $s$-orthostochastic matrices is contained in $\bB_d$ and spans all of $\bB_d$ as $s \to \infty$. Other works have investigated further properties of the spectral span of the Birkhoff polytope \cite{kim_conjectures_2022} and the geometry of $s$-orthostochastic matrices in the quantum setting \cite{Karol_Zyczkowski_2003, Korzekwa_2018, Shahbeigi_2021}.
These theoretical foundations underlie our choice of parameterization.
To our knowledge, we are the first to apply this construction in a machine learning setting and, more specifically, to learned residual stream routing in deep networks.

\section{go-\mHC: Generalized Orthostochastic Manifold-Constrained Hyper-Connections}
\label{sec:method}

\subsection{Hyper-Connections}

Following the illustration in figure~\ref{fig:one} and prior work \cite{zhu2025hyper,xie2025mhc}, we define a single Hyper-Connection block's $d$-stream state $\vx_l \in \R^{d \times B \times L \times D}$ with the recursion relation including a residual block $\cF$:
\begin{equation}
  \vx_{l+1} = \cH_l^{\text{res}}\, \vx_l + (\cH_l^{\text{post}})^\top\, \cF\!\left(\cH_l^{\text{pre}}\, \vx_l,\, \mathbf{W}_l\right),
  \label{eq:hc}
\end{equation}
where $\cH_l^{\text{res}} \in \R^{d \times d}$ mixes residual streams, $\cH_l^{\text{pre}} \in \R^{1 \times d}$ aggregates streams into the branch input, and $\cH_l^{\text{post}} \in \R^{1 \times d}$ distributes branch output back to streams. With $\vx_{0}$ defined as $d$ repetitions of the initial input as done in \citet{zhu2025hyper}. In the limit of $\cH_l^{\text{res}}=\cH_l^{\text{pre}}=\cH_l^{\text{post}}=[1]$, this is the typical single-stream residual connection.

To allow for a dynamic topology of the network, the $d$-stream mixing, readout and write-in terms are
allowed to be dynamic with learned parameters $\{(\alpha^s_l, \mathbf{W}^s_l, \mathbf{b}^s_l)\}_{s \in \{\text{pre, post, res}\}}$:

\begin{align}
    &\mathbf{x}_l^{\prime} = \text{RMSNorm}(\mathbf{x}_l), \\
    &\cH_l^{\text{pre}} = \text{sigmoid}\left( \alpha_l^{\text{pre}} \mathbf{x}_l^{\prime} \mathbf{W}_l^{\text{pre}} + \mathbf{b}_l^{\text{pre}} \right), \\
    &\cH_l^{\text{post}} = 2\cdot\text{sigmoid}\left( \alpha_l^{\text{post}} \mathbf{x}_l^{\prime} \mathbf{W}_l^{\text{post}} + \mathbf{b}_l^{\text{post}} \right), \\
    &\cH_l^{\text{res}} = \cP\left( \text{mat} \left( \alpha_l^{\text{res}} \hat{\mathbf{x}}_l' \mathbf{W}_l^{\text{res}} + \mathbf{b}_l^{\text{res}} \right) \right)\label{eq:hres_proj_def}
\end{align}

In \mHC, the manifold constraint operator $\cP=\cP_{\text{SK}}$ is realized via 20 iterations of the Sinkhorn-Knopp algorithm, where rows and columns are normalized iteratively.

As done in \mHC-lite and KromHC, we opt to keep the above formulation
and choose to only modify $\cH_l^{\text{res}}$ in equation~\ref{eq:hres_proj_def} to be an exact projection.
We instead take $\cP=\cP_{\text{go}}$ as a map from $\frac{ds(ds-1)}{2}$ parameters generated by our learned parameters $(\alpha^{\text{res}}_l, \mathbf{W}^{\text{res}}_l, \mathbf{b}^{\text{res}}_l)$ to the target $d\times d$ mixing matrix using the algebraic theory of generalized orthostochastic matrices.

\subsection{Generalized Orthostochastic Matrices}
\label{sec:gen_ortho}

A matrix $\mathbf{B} \in \bB_d$ is called \emph{orthostochastic} if $\mathbf{B}_{ij} = \mathbf{Q}_{ij}^2$ for some orthogonal matrix $\mathbf{Q} \in \cQ(n)$. When $n \geq 3$, the set of orthostochastic matrices are proper non-convex subsets of $\bB_d$ \cite{convexity_3by3, alg_and_geo_inside_Bn, nechita2023generalized}.

For $d, s \in \mathbb{Z}^+$, let $\mathbf{Q} \in \cQ(ds)$ be viewed as a $d \times d$ block matrix with $s \times s$ blocks $\mathbf{Q}_{ij}$. Define:
\begin{equation}
    \Phi_{d,s}(\mathbf{Q})_{ij} = \frac{1}{s}\,\|\mathbf{Q}_{ij}\|_F^2 = \frac{1}{s}\sum_{k=1}^{s}\sum_{l=1}^{s} (\mathbf{Q}_{ij})_{kl}^2.
  \label{eq:phi}
\end{equation}
\textbf{Proposition} (Nechita et al., Theorem 2.8 \cite{nechita2023generalized}): $\Phi_{d,s}(\mathbf{Q}) \in \bB_d$ for all $\mathbf{Q} \in \cQ(ds)$ and all $s$, and the set $\{\Phi_{d,s}(\mathbf{Q}) : \mathbf{Q} \in \cQ(ds)\}$ converges to $\bB_d$ as $s \to \infty$.

The parameter $s$ determines the rank of the block-wise Frobenius projection, where larger $s$ provides a more faithful representation of the polytope interior. At $s=1$, we recover the orthostochastic set by definition. The geometry of the orthostochastic set is a concave cycloid that meets the permutation vertices of $\mathsf{B}_d$ without completely occupying the interior of $\mathsf{B}_d$ when $d>2$. Theorem 2.8 guarantees convergence to span the full polytope $\bB_d$ as $s \to \infty$ with monotonically increasing volume. Based on this and results in \cite{schwarz1980infinite} we choose to primarily focus on $s=2$ as our default variant of go-\mHC, but we show comparisons with $s=1$ and $s=3$ whenever mentioned. We further discuss the implications of $s$ on geometry in appendix~\ref{app:geometry}.

\begin{figure}[hbt!]
    \centering
    \vspace{-0.3cm}
    \includegraphics[width=0.75\linewidth]{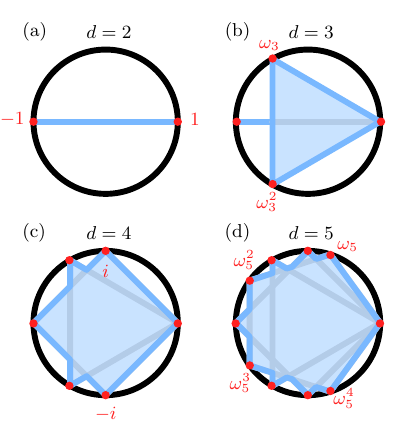}
    \vspace{-0.35cm}
    \caption{Illustration of the Karpelevi\v{c} region -- the spectrum of stochastic matrices -- for $\mathsf{B}_d$ with $d=2\text{ to } d=5$. The black circle corresponds to the unit disc and the red points are the eigenvalues $\omega$ of the $d\times d$ permutation matrices contained in the polytope. The affine space defined by doubly stochastic matrices is strictly contained within the volume of the blue region, except for in $\mathsf{B}_2$ where the space is a line. The spectrum of $\mathsf{B}_d$ is the union of the spectrum of $\mathsf{B}_{d-1}$ and the polygon with $d$ sides inscribed on the unit disc.}
    \label{fig:spectrum_illustration} 
\end{figure}

\subsection{Learning the $s$-Orthostochastic Manifold}

We write out our full projection, replacing equation~\ref{eq:hres_proj_def}, below:
\begin{align}
    &\mathbf{A}_l = \text{skew} (\alpha_l^{\text{res}} \mathbf{x}_l^{\prime} \mathbf{W}_l^{\text{res}} + \mathbf{b}_l^{\text{res}} ),\\
    &\mathbf{Q}_l = (\mathbf{I}-\mathbf{A}_l)(\mathbf{I}+\mathbf{A}_l)^{-1}, \label{eq:cayley_q}\\
    &\cH_l^{\text{res}} = \Phi_{d, s}(\mathbf{Q}_l), \label{eq:Hres}
\end{align}
where the function $\text{skew}$ maps vectors of size $\frac{ds(ds-1)}{2}$ to a $ds\times ds$ skew-symmetric matrix
$\left ( \mathbf{A}_l^\top = -\mathbf{A}_l\right)$. Equation~\ref{eq:cayley_q} makes use of the Cayley transform to map
$\mathbf{A}_l$ to the space of special orthogonal matrices $\left ( \mathbf{Q}_l^\top \mathbf{Q}_l=\mathbf{I}\right)$. We can then use $\Phi_{d, s}$ from equation~\ref{eq:phi} to map $\mathbf{Q}_l$ to an exactly doubly stochastic matrix $\cH_l^{\text{res}}\in\mathsf{B}_d$ in equation~\ref{eq:Hres}.

The Cayley transform specifically generates matrices with $\det (\mathbf{Q}_l)=+1$, which 
is sufficient for full coverage since the orthostochastic transform takes the $||\cdot ||_F^2$ of the block matrices and removes information about reflections, keeping only the positive determinant solutions. This is also advantageous due to the path-connectedness of the special orthogonal group, the group of all ``pure'' rotations in $\mathbb{R}^d$ -- more on this in appendix~\ref{app:proof_of_connectedness}.

By leveraging the Cayley transform to map learned parameters onto $\mathbf{Q}_l$, the necessity for a softmax activation on the inputs is bypassed. The Cayley transform is inherently norm-preserving which ensures numerical stability without the need for the additional normalization step. We suspect that this replacement of softmax with the Cayley transform is one of the reasons that go-\mHC\ accelerates convergence, due to its avoidance of the saturation regions of the exponential function (see results in section~\ref{sec:convergence}). While additionally mitigating potential vanishing gradients (appendix~\ref{app:mhclite_vanishing_gradients}) and maintaining a well-conditioned Jacobian throughout training. These vanishing gradients trade-off how close mixing matrices near the boundary of the polytope can be learned and the ability to ``unlearn them'' and return to the Barycenter.

While \citet{xie2025mhc} use the Amax Gain Magnitude (AGM) to quantify the worst-case amplification for $\cH_{l\to L}^{\text{res}} = \prod_{i=1}^{L-l} \cH_{L-i}^{\text{res}}$:
\begin{align}
  \text{AGM}^{\text{fwd}}(\cH_{l\to L}^{\text{res}}) = \max_{i} \sum_j |(\cH_{l\to L}^{\text{res}})_{ij}|,\\
  \text{AGM}^{\text{bwd}}(\cH_{l\to L}^{\text{res}}) = \max_{j} \sum_i |(\cH_{l\to L}^{\text{res}})_{ij}|.
\end{align}

We don't rely on this metric due to the exactness of $\cP_{\text{go}}$:
by definition, $\max_{i} \sum_j |(\cH_{l}^{\text{res}})_{ij}|=\max_{j} \sum_i |(\cH_{l}^{\text{res}})_{ij}|=1$. Moreover, since $\bB_d$ is closed under matrix multiplication, the product $\cH_{l\to L}^{\text{res}}$ remains doubly stochastic regardless of depth. Since go-\mHC, \mHC-lite, and KromHC all perform equally well on the AGM metric due to their exactness, we use two additional metrics to benchmark our method: expressivity and convergence.

\section{Expressivity}
\begin{figure}[t!]
    \centering
    \includegraphics[]{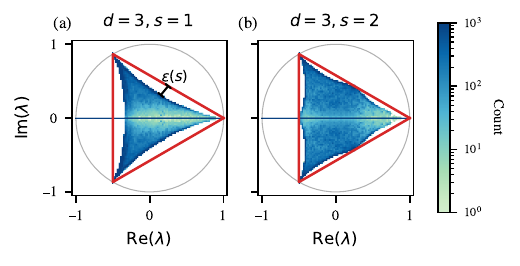}
    \caption{Histogram of the spectral reach of a toy model implementing the map $\cP_{\text{go}}$ to $\mathsf{B}_3$ ($d=3$) for different values of $s \in \{1, 2\}$ on random targets.
        \textbf{(a)} When $s=1$, the $d=3$ model is able to learn maps with full spectral coverage within the hypocycloid--covering the triangular region (Karpelevi\v{c} region of $\mathsf{B}_3 \setminus \mathsf{B}_2$) up to error $\epsilon(s=1)$, as well all real eigenvalues on the interval $[-1, 1]$ (Karpelevi\v{c} region of $\mathsf{B}_2$).
        \textbf{(b)} $s=2$ covers a larger region, with a smaller gap in the spectra $\epsilon(s=2)<\epsilon(s=1)$. 
        }
    \label{fig:spectra_numerics}
\end{figure}

Motivated by the effect of bounding the spectrum on having stable gradients (see appendix~\ref{app:gradient_stability_from_spectral_radius}) \cite{xie2025mhc, yang2026mhclite}, we choose to evaluate the expressivity of
the different methods by using the spectrum of matrices constructed with $\cP_{\mathsf{M}}$ as a surrogate for expressivity -- quantifying how many unique matrices can be spanned using their set of eigenvalues. We rely on surrogate measures for the expressivity of doubly stochastic matrices due to the complicated geometry of the Birkhoff polytope and given that finding the exact volume of $\mathsf{B}_d$ is an open problem \cite{polytope_volume, beck2005ehrhartpolynomialbirkhoffpolytope, kim_conjectures_2022}. Formally, we define the spectral reach of $\mathbf{H}^{\text{res}}$ matrices constructed using a given method $\mathsf{M}$ as:
\begin{equation}
    \{\lambda \in \mathbb{C} \;\vert\; \mathbf{H}^{\text{res}}\vec v = \lambda \vec v,\; \mathbf{H}^{\text{res}} = \cP_{\mathsf{M}}(\cdots),\; \vec v\neq \vec 0\}
\end{equation}

\subsection{Spectral Reach of $\mathsf{B}_d$ Subsets}
\label{sec:spectral_analysis}

A mapping $\mathsf{M}$ with high spectral reach generates matrices spanning a larger volume of the Karpelevi\v{c} region—the compact subset of the unit disk containing all possible eigenvalues of stochastic matrices. This region's boundary is defined by algebraic curves connecting the roots of unity \cite{Karol_Zyczkowski_2003, kim_conjectures_2022}:
$$\{e^{i2\pi\frac{j}{d'}} \mid j \in [1, d'], d' \leq d\}$$

\begin{figure*}[t!]
    \centering
    \includegraphics[width=\textwidth]{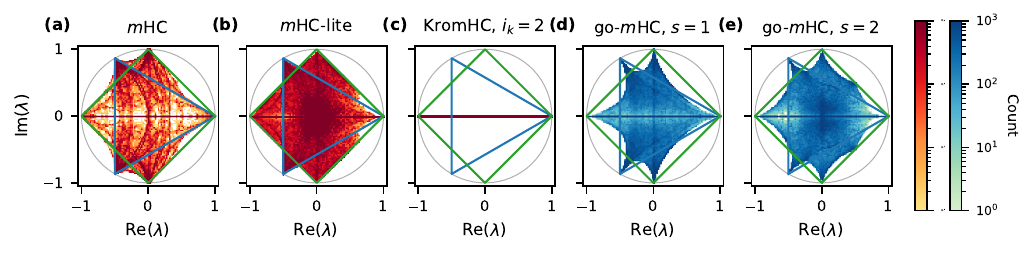}
    \caption{Histogram of the spectral reach of a toy model implementing the maps $\cP_{\text{SK}}$, $\cP_{\text{lite}}$, $\cP_{\text{krom}}$, and $\cP_{\text{go}}$ on random targets in  $\mathsf{B}_4$.
        The blue and green lines in \textbf{(a-e)} are the boundary of the spectrum of the $\mathsf{B}_3 \setminus \mathsf{B}_2$ and $\mathsf{B}_4 \setminus \mathsf{B}_3$ polytopes respectively,
        as illustrated in figure~\ref{fig:spectrum_illustration}.
        \textbf{(a)} \mHC (implemented via $\cP_{\text{SK}}$) contains points outside of $\mathsf{B}_4$ (this effect is exaggerated for illustration via initial conditions with large variance).
        \textbf{(b)} \mHC-lite shows good expressivity, filling the region near the boundary almost completely, without extending past it.
        \textbf{(c)} KromHC fails to represent the complex region inside of $\mathsf{B}_4$ -- since the ``shadow'' is $1$-dimensional, it fills an infinitesimal slice of the full polytope.
        \textbf{(d)} For $s=1$, we see go-\mHC represent a sizeable portion of the region, although it cannot represent all matrices near the boundary of the polytope.
        \textbf{(e)} With $s=2$, go-\mHC begins to represent larger regions of $\mathsf{B}_4$, filling the region more densely.}
    \label{fig:all_spectra}
\end{figure*}

An illustration of the Karpelevi\v{c} region for $d=2\text{ to } d=5$ is shown in figure~\ref{fig:spectrum_illustration}. If a method's spectrum is clustered within a small sub-region, it is poorly expressive; conversely, a method that ``fills'' the Karpelevi\v{c} region captures the diverse dynamics represented by the polytope's extremal points. The spectral reach of \mHC, KromHC, and go-\mHC\ can all be found in ~\ref{fig:spectra_numerics}\&\ref{fig:all_spectra}.

\subsubsection{Eigenvalue Coverage of $s$-Orthostochastic Matrices}

The eigenvalue distribution of $d \times d$ $s$-orthostochastic matrices is fundamentally constrained by their membership in the Birkhoff polytope $\mathsf{B}_d$ and residence within the Karpelevi\v{c} region. While $1$-orthostochastic matrices exhibit a restricted, hypo-cycloidal boundary, increasing the parameter $s$ provides a progressively denser approximation of the full $\mathsf{B}_d$ \cite{nechita2023generalized}. Consequently, as $s$ increases, the spectral support of $s$-orthostochastic matrices expands to fill the Karpelevi\v{c} region, approaching the density of general doubly stochastic matrices, as demonstrated in figure~\ref{fig:spectra_numerics}.

\subsubsection{The Spectral Limitation of Kronecker Products}

While KromHC is the most efficient of the proposed methods, it imposes a symmetry on the $d$-streams that significantly prunes the space of potential mixing terms in $\cH^{\text{res}}$, as shown in figure~\ref{fig:all_spectra}c. We demonstrate that KromHC parameterizes only a sparse subset of $\mathsf{B}_d$ for $d>2$ by analyzing the spectral reach of $k$-fold Kronecker products. Specifically, we prove that the spectral space of a $k$-fold Kronecker product of doubly stochastic matrices is invariant to the index $k$, effectively restricting the expressive reach of the resulting constraint to that of its constituent base manifold. A formal proof of this invariance is provided in Appendix~\ref{app:kron_kfold_2_proof}.

To illustrate this expressivity bottleneck, consider the $4 \times 4$ case where $\mathbf{A}, \mathbf{B} \in \mathsf{B}_2$. Their individual spectra are defined as $\{1, \lambda_A\}$ and $\{1, \lambda_B\}$ for $\lambda \in [-1, 1]$. The spectrum of the Kronecker product $\mathbf{A} \otimes \mathbf{B}$ is thus $\{1, \lambda_A, \lambda_B, \lambda_A \lambda_B\}$. Crucially, since the product $\lambda_A \lambda_B$ remains real and within the interval $[-1, 1]$, the resulting eigenvalues are strictly contained within the spectral hull of $\mathsf{B}_2$ (see Figure~\ref{fig:all_spectra}b for emperical results and a full proof in \ref{app:kron_kfold_2_proof}). Due to this closure, the spectral reach of the Kronecker-parameterized manifold is invariant to the number of products and fails to expand into the broader Karpelevi\v{c} region available to the full Birkhoff polytope, representing increasingly smaller fractions of the full polytope. By volume, this is a negligible portion of the full $\mathsf{B}_d$ when $d > i_k$, and we show emperically in appendix~\ref{app:kron_kfold_2_app_err} that this translates to a $\cO(d)$ additional error in the minimum achievable loss compared to \mHC-lite and go-\mHC.

The collapse in spectral volume for Kronecker-based transforms represents a significant loss in functional diversity. Specifically, these constraints fail to represent cyclic permutations of order greater than $i_k=2$ (e.g., $1\to 2\to 3\to 1$), effectively preventing the model from learning complex periodic dependencies. Furthermore, the restriction to real eigenvalues in the $2 \times 2$ base case implies that the system can only model pure diffusion (information spreading) without advection (directed rotation of information), severely limiting the directional flow of signals within the reservoir.

\subsection{Parameter Complexity}
go-\mHC's per-layer parameter complexity is similar to that of \mHC, and much lower than that of \mHC-lite.
We note that it is $\cO(dC\cdot d^2s^2)$, similar to \mHC's $\cO(dC\cdot d^2)$ and much better than the exact solution in
\mHC-lite that requires $\cO(dC\cdot d!)$. KromHC trades off expressivity for a much smaller parameter
count that is $\cO(dC\cdot d)$. See table \ref{tab:param_count_table} for exact expressions and figure~\ref{fig:param_count}
for a visual comparison for residual streams of size $d\in [2,32]$.

\begin{figure}[t]
    \centering
    \includegraphics[width=0.85\linewidth]{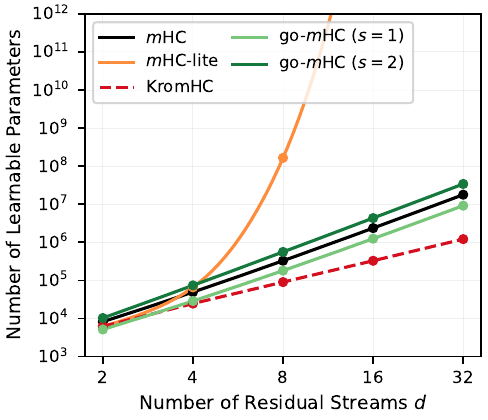}
    \caption{Comparing the total number of learnable parameters in \mHC\  with its exact variants. \textbf{go-\mHC} (green), same scaling as \mHC\ regardless of $s$, with $s=1$ requiring less parameters due to the skew-symmetry. \textbf{\mHC-lite} (orange) scales with $\cO(d!)$ despite its exactness and completeness over $\mathsf{B}_d$, blowing up beyond $d=8$ and becoming intractable. \textbf{KromHC} (red) scales the fastest $\cO(d\log d)$ faster than \mHC\ by trading off expressivity and imposing a strong symmetry on residual stream interactions. KromHC relies on the factorization of $d$ and we therefore only analyze the best case parameter count when $d$ is a power of 2.}
    \label{fig:param_count}
\end{figure}

\begin{table}[ht]
    \centering
    \caption{Total Parameter Counts}
    \label{tab:param_count_table}
    \begin{tabular}{ll}
        \toprule
        \textbf{Method} & \textbf{Parameter Count} \\
        \midrule
        \mHC            & $\mathcolor{blue}{(dC + 1) d^2} + 2d^2C + 2d + 3$ \\
        go-\mHC\ ($s$)   & $\mathcolor{blue}{(dC + 1) \frac{ds(ds - 1)}{2}} + 2d^2 C + 2d + 3$ \\
        \mHC-lite       & $\mathcolor{blue}{(dC + 1) d!} + 2d^2 C + 2d + 3$ \\
        KromHC         & $\mathcolor{blue}{(dC + 1) \sum_{k=1}^{K} (i_k!)} + 2d^2 C + 2d + 3$ \\
    \end{tabular}
\end{table}

We note that KromHC and go-\mHC\ are complementary methods that can be used simultaneously. Using both methods, we can learn tensor products of
generalized orthostochastic matrices, which would allow KromHC to gain back some expressivity (using larger product sizes, $i_k$) at quadratic expressivity cost instead of factorial:
$\mathcolor{blue}{(dC+1)\cdot \sum_{k=1}^{K} \frac{1}{2}si_k(si_k-1)}$. This allows the expressivity to become a tunable parameter that scales as $\cO(s^2i_k^2\cdot \log_{i_k}d)$ instead of $\cO(i_k!\cdot\log_{i_k}d)$ -- allowing for greater flexibility across a variety of compute-limited applications.

\subsection{FLOP Complexity}

We additionally track the additional overhead of our method on top of Hyper-Connections. In unconstrained HC, experiments show negligible incurred FLOP complexity ($\approx 0.15\%$ increase) during a forward pass in a 7B parameter model \cite{zhu2025hyper}. Projecting onto a manifold with $\mathcal{P}_\mathsf{M}$ incurs an additional FLOP complexity due to the procedure introduced by $\mathsf{M}$. In \mHC, this is significant ($\approx 8\%$ additional time overhead for a 27B model) even after the development of custom kernels and major infrastructure improvements to perform the Sinkhorn-Knopp iterations. KromHC and \mHC-lite (in the small $d$ regime) both show significant speed ups due to the removal of the iterative updates. We here calculate the FLOP complexity of our method to show that it is superexponentially better than \mHC-lite and incurs a small relative overhead compared to KromHC for the increased expressivity. For example, for $d=16$, the flop complexity of go-\mHC\ is $2$ orders of magnitude slower than KromHC, compared to the $16$ orders of magnitude imposed by \mHC-lite. It is also approximately the same order of magnitude in complexity as \mHC\ with SK iterations due to $\cO(d)\approx \cO(S)$.

In general, the model performance improvement is larger than the incurred computational
cost; however, when considering scaling the number of residual streams to allow for richer
topology, the complexity begins to set a practical limit -- we hope to provide a 
middle ground method that provides a good balance of incurred FLOP cost while maintaining
rich enough expressivity to allow for scaling.

\begin{table}[ht]
    \centering
    \caption{Theoretical FLOP complexity per layer. FLOP counts demonstrate that go-\mHC\ provides a practical middle ground for scaling residual streams, bypassing the superexponential cost of \mHC-lite and the iterative overhead of SK iterations in \mHC, with $S$ being the number of SK iterations. Nominal values for the constants below are: $S=20, i_k=2$, and $s\in\{1, 2\}$. See appendix \ref{app:complexity_full} for our full FLOP estimates.}
    \label{tab:flop_complexity}
    \begin{tabular}{ll}
        \toprule
        \textbf{Method} & \textbf{FLOP complexity (per layer)} \\
        \midrule
        \mHC            &  $\cO\left(C\cdot S\cdot d^2\right)$ \\
        go-\mHC\ ($s$)   & $\cO\left(C\cdot s^3 \cdot d^3\right)$ \\
        \mHC-lite       & $\cO\left(C\cdot d^2\cdot d!\right)$ \\
        KromHC         & $\cO\left(C\cdot\log_{i_k}(d)\cdot i_k^2\cdot (i_k!) + C\cdot d^2\right)$ \\
    \end{tabular}
\end{table}

\section{Convergence}
\label{sec:convergence}

The performance and convergence properties of $\cP_\mathsf{M}$ are fundamentally governed by the geometry of the underlying loss landscape introduced by $\mathsf{M}$ and the stability of the associated gradient signals. In the context of \mHC\, this indicates that any target matrix should be reachable regardless of initial conditions using finite optimization iterations. We do expect go-\mHC\ to have learnable gradients, be convergent and fully connected (able to get to any target matrix). See appendix \ref{app:proof_of_connectedness} for more details on this claim and associated proofs.

Not all matrix parameterizations are learnable and fully connected. \mHC-lite should generally perform well by choosing a good initialization, but we suspect that this does not mitigate the fact that it may become very difficult to relearn an $n$-stream transform after having previously converged as the model trains for longer due to the potential of vanishing gradients (a heuristic example is from needing to remap streams due to generalization). Numerics and further discussion can be found in the appendix \ref{app:mhclite_vanishing_gradients}.

\subsection{Synthetic Stream Mixing Performance}

To establish a baseline for performance, we first conduct a synthetic ``toy model'' experiment designed to evaluate the reconstruction capabilities of
the proposed manifold parameterizations in a controlled environment. 
The framework generates a set of ground truth doubly stochastic matrices $\mathbf{T}_i$ and a dataset of dense activations in the residual stream $\mathbf{X}_j\in \mathcal{X}$. We then calculate noisy targets $\mathbf{Y}_{i,j}$ representing
a target latent representation within the residual stream:
\begin{equation}
    \mathbf{Y}_{i,j}=\mathbf{T}_i\mathbf{X}_j + \mathbf{\xi}
\end{equation}

with $\mathbf{\xi} \sim \epsilon\cdot \mathcal{U}(0, 1)$ and noise magnitude $\epsilon$.

We assess the fidelity of various manifold parameterizations by minimizing the mean squared error between $\mathbf{Y}$ and $\mathbf{\tilde Y}$ calculated from the learned residual mapping $\cH^{\text{res}}_i$:
\begin{align}
    &\mathbf{\tilde Y} =\cH^{\text{res}}_i\mathbf{X}_j \\
    &\mathcal{L}_i=\sum_j|| (\cH_i^{\text{res}}X_j) - (\mathbf{T}_iX_j+ \mathbf{\xi})||^2_2
\label{eq:toy_model_full}
\end{align}

We investigate the regime of random dense vectors $\mathbf{X}_j$ in this section to mimic the residual stream's dense non-priveleged basis \cite{transformer_circuits, lawson2025residual, brown2023privilegedconvergentbasesneural}, however, appendix~\ref{app:sparsity} showcases that the introduction of sparsity in our toy model does not impact our claims.

\begin{figure}[ht]
    \centering
    \includegraphics[width=\linewidth]{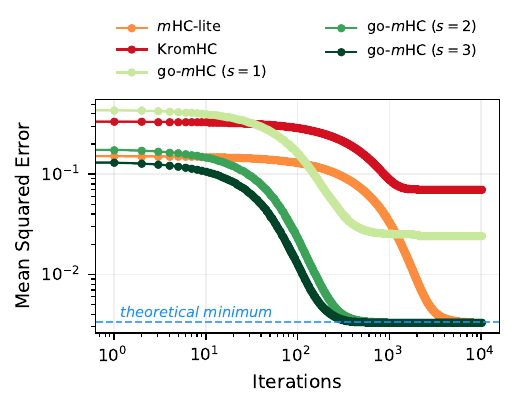}
    \caption{A representative sample of a training curve on a synthetic $4$-stream mixing task with a fixed learning rate $\eta=10^{-3}$, a batch size of $C=64$, dataset size $|\mathcal{X}|=100$ and noise mangitude $\epsilon=10^{-1}$. We see that both \textbf{\mHC-lite} (orange) and \textbf{KromHC} (red) take longer to converge than \textbf{go-\mHC} (greens). We also note that go-\mHC\ with $s\geq2$ saturates at the same loss as that of \mHC-lite, which happens to be the theoretical minimum achievable loss due to the noise term in equation \ref{eq:toy_model_full}. KromHC has the worst performance due to its expressivity-speed trade-off and go-\mHC\ with $s=1$ offers a loss in-between that of KromHC and \mHC-lite or its $s\geq 2$ counterparts.}
    \label{fig:loss_trajectory}
\end{figure}

Figure~\ref{fig:loss_trajectory} showcases a representative sample of training loss
curves exhibited by \mHC-lite, KromHC and go-\mHC, illustrative of the performance trends observed in general. go-\mHC\ generally (1) reaches the lowest observed losses, and (2) begins descending and reaches this loss almost $10$ times faster than previous methods. We also see that due to the symmetry imposed by KromHC, it consistently sees higher errors in our toy models for $d>2$. Section \ref{app:symmetry_breaking} numerically investigates additional symmetry breaking terms from read/write projections and showcases that our conclusions are invariant to their inclusion.

\begin{figure}[t]
    \includegraphics[width=\linewidth]{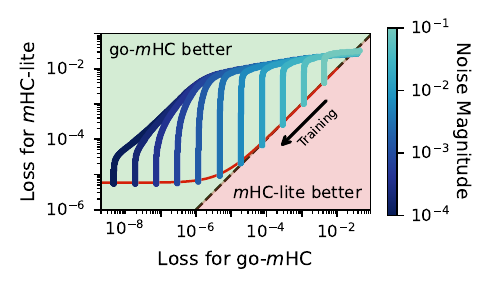}
    \caption{Phase-space comparison of training trajectories under varying noise regimes using the toy model in appendix~\ref{app:learned_dist_model}. We plot the MSE of \mHC-lite against go-\mHC\ across four orders of magnitude of noise magnitude (color-coded). Each curve represents a training trajectory starting from the top-right (high loss) and progressing toward the bottom-left (convergence). The dashed diagonal ($y=x$) separates regions of relative performance; trajectories residing in the upper-left green shaded region indicate that go-\mHC\ achieves a lower loss than \mHC-lite at equivalent training stages. As noise magnitude decreases, go-\mHC\ maintains a consistent performance advantage, converging to the optimal loss floor significantly faster than the lite variant. See figure~\ref{fig:convergence_comparison_noise_dependence} for the individual loss curves.}
    \label{fig:convergence_MHC_lite_w_adam}
\end{figure}

We additionally see this trend in figure~\ref{fig:convergence_MHC_lite_w_adam}, where throughout training, go-\mHC\ consistently has better loss regardless of the level of noise added to features. Figure~\ref{fig:loss_vs_mat_size} shows that this trend is sustained over various number of streams $d$ and across multiple epochs. More ablations, including additional optimizer configurations, parameter sweeps, and models of latent data \& target transformations are offered in appendix~\ref{app:sensitivity_main}.

\begin{figure}[t]
    \centering
    \includegraphics[width=\linewidth]{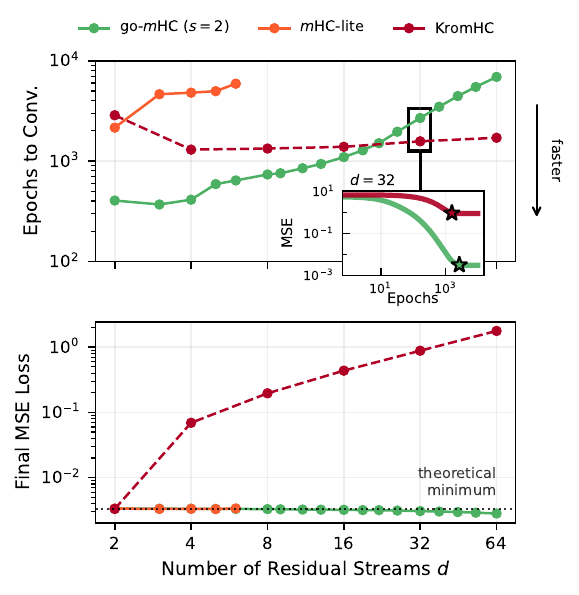}
    \caption{Comparison of convergence speed and final loss across different \mHC\ architectures as a function of the number of residual streams $d$.
    \textbf{Top Panel}: Convergence efficiency measured by Epochs to Convergence across varying $d$. While KromHC (red, dashed) show relatively stable or slightly increasing epoch counts for small $d$, go-\mHC\ (green) exhibits a steady increase in training time as the number of residual streams grows. An inset displays the loss trajectory for $d = 32$, highlighting that even though go-\mHC\ takes longer, it achieves a significantly lower error floor than KromHC despite the longer convergence time. We did not generate \mHC-lite data points beyond $d=7$ due to its factorial scaling, but suspect it follows a similar trend to that of go-\mHC.
    \textbf{Bottom Panel}: Scalability of the Final MSE Loss. Both go-\mHC\ and \mHC-lite consistently reach the theoretical minimum (dotted line) across the tested range. In contrast, KromHC suffers from severe performance degradation as $d$ increases, with the final loss diverging significantly from the optimal baseline, we find that it scales as $\cO(d)$ regardless of $i_k$ in appendix~\ref{app:kron_kfold_2_app_err}.}
    \label{fig:loss_vs_mat_size}
\end{figure}

We note a potential symmetry-breaking effect introduced by the projection $\mathcal{H}^{\text{pre}}$ as it interfaces with the computational layers $\mathcal{F}$. However, as detailed in appendix~\ref{app:symmetry_breaking}, go-\mHC’s improved expressivity and convergence remain robust to these interface terms.

\subsection{Validating with Tiny Language Model}
\label{sec:experiments}

We validate our observations in a tiny language models by comparing the convergence of
go-\mHC\ with that of \mHC-lite and KromHC in a 30M parameter model with $6$ layers and $6$ heads in nanoGPT \cite{nanogpt, tinystories}. We demonstrated similar performance using any of \mHC-lite, KromHC and go-\mHC\ on TinyStories and Character-level datasets, shown in figure~\ref{fig:validation_tiny_lm}. We also ran our method with $d=8$ (as shown in \ref{fig:training_trace_d8}), a stream size that would be prohibitively large to run with \mHC-lite due to the factorial scaling of its parameter count and complexity as shown in figure~\ref{fig:param_count}. This validates our method in principle, and demonstrates it's potential scalability over \mHC-lite, but requires additional verification in larger models to more rigourously evaluate and compare the methods. 

\begin{figure}
    \centering
    \includegraphics[width=\linewidth]{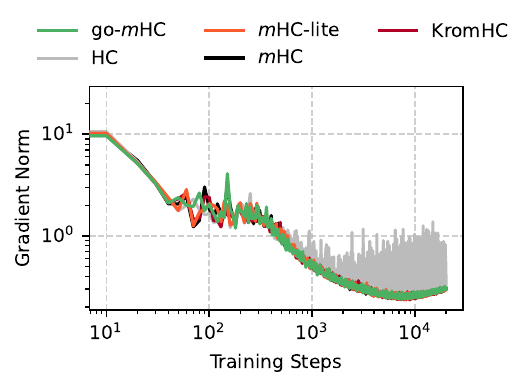}
    \caption{Gradient norm of different Hyper-Connection (HC) variants. We show the global gradient norm for a 30M parameter GPT-style transformer, trained on TinyStories. Our training runs indicate that all \mHC\ variants exhibit similarly convergent behavior with well-defined gradients, with a substantial improvement over HC, which suffers from unstable gradients. The training loss values remain within a negligible margin of one another (appendix~\ref{app:tiny_lm_training}).}
    \label{fig:validation_tiny_lm}
\end{figure}

We evaluate go-\mHC\ against baselines using the LLM-as-a-Judge framework with GPT-4.1, to grade the grammar, creativity, and consistency of stories generated by the models. Our evaluation results are listed in table~\ref{tab:reg_70}. We generate $70$ stories from each model with consensus of $3$ independent trials. We use the independent trials to calculate the intraclass correlation coefficient, where we see fairly good reliability. We tested smaller variants of GPT-4.1 and inter-model reliability in appendix~\ref{app:judges}.

Our method achieves parity with established baselines, suggesting our method scales from our toy model as intended. Due to compute limitations, we restrict our experiments to a tiny model, however, prior work \cite{zhou2026kromhc} suggests that model performance improves as the number of residual streams is scaled, so we instead attempt to demonstrate that our model achieves equivalent performance to prior work without the ability to claim advantage outside of the small model scenario. We suspect that the ability of go-\mHC\ to learn the full polytope would be more advantageous in the limit of deeper models that require multi-step reasoning with indirect composition \cite{lindsey2025biology} where computation may be more distributed across layers.
\input{table_llm_judge}

\section{Discussion}

In this work, we introduced an exact parameterization of the Birkhoff polytope using generalized orthostochastic matrices and demonstrated its application to manifold-constrained hyper-connections. By leveraging the algebraic properties of generalized orthostochastic matrices, go-\mHC\ addresses a long-standing trade-off between computational scalability and expressive reach in parameterizing doubly stochastic matrices. Our spectral analysis and loss benchmarks demonstrate that go-\mHC\ significantly outperforms existing exact methods like KromHC by filling the Karpelevi\v{c} region more substantially. This allows the model to capture a broader range of mixing dynamics and maintain expressivity as the residual stream is scaled. Furthermore, our empirical results on synthetic tasks---showing up to a $10\times$ speedup in convergence---suggest that the orthogonal matrix parameterized path within $\mathsf{B}_d$ provides a more favorable optimization landscape than the convex combinations used in \mHC-lite and KromHC. This efficiency, combined with $\cO(d^3)$ scaling, provides a viable technical path for increasing the number of residual streams as a new scaling dimension for deep networks.

Despite these advantages, our approach has several limitations. Our experiments were primarily conducted with synthetic tasks and scale-limited language models. While these results establish the method's correctness and relative advantages, the degree to which these convergence benefits translate to real-world performance remains an open question. Additionally, while $\cO(d^3)$ complexity is a significant improvement over \mHC-lite’s factorial scaling, it introduces a nontrivial overhead compared to KromHC’s $\cO(d^2)$ cost for large $d$. Using go-\mHC\ in tandem with Kronecker-factored methods remains a promising strategy for bridging this complexity gap while maintaining expressivity. Finally, the dependence on the Cayley transform introduces a matrix inversion bottleneck; future work could explore alternative mappings to the orthogonal group, such as Householder reflections or Hurwitz angles, to further optimize hardware-level efficiency.

\section*{Author Contributions}

TD and SDG jointly developed the codebase, executed the large-scale simulations, and performed all toy model training, including hyperparameter optimization and sensitivity analysis. The manuscript was written collaboratively, with both authors contributing equally to the text, figures, and the final synthesis of results.

TD conceptualized the research and established the theoretical framework, including the formulation of toy models and formal analysis. Specifically, TD provided the mathematical derivations for parameter and computational complexity, manifold projections, spectra calculations, and training dynamics. The training and analysis of the small language models was also done by TD.

\section*{Acknowledgments}

The authors would like to thank Ayodeji Lindblad for insights that helped our understanding of the geometry and shadows of the Birkhoff polytope in higher dimensions.

We would also like to express our appreciation to the authors of the {\mHC} paper; their insightful work was the primary catalyst for our interest in manifold-constraints and directly motivated the research questions explored in this study. Additionally, we thank the maintainers of {nanoGPT} and the authors of {\mHC-lite} and {KromHC} for publishing the codebase that helped facilitate our verification experiments in tiny language models.

\section*{Reproducibility}
\label{sec:reproducibility}
To ensure reproducibility, we provide the complete source code, pre-trained models, and instructions to reproduce the main experimental results in our GitHub repository: \url{https://github.com/itstorque/go-mHC}.

\clearpage

\bibliographystyle{plainnat} 
\bibliography{refs}
\clearpage
\appendix
\ifshowcontent

    \section*{Appendix}
    \vskip 1em
    
    \begin{enumerate}[label=\textbf{\Alph*.}, leftmargin=0cm, itemsep=2ex]
    
        \item \textbf{Sensitivity Analysis and Model Variations} \dotfill \pageref{app:sensitivity_main}
        \begin{enumerate}[label=\Alph{enumi}.\arabic*, leftmargin=0cm, labelsep=5pt]
            \item Number of Streams \dotfill \pageref{app:loss_curve_vs_d}
            \item Sparsity \dotfill \pageref{app:sparsity}
            \item SGD instead of Adam \dotfill \pageref{app:convergence_MHC_lite_w_sgd}
            \item Alternative Sampling Distributions for Doubly Stochastic Matrices \dotfill \pageref{app:target_perms}
            \item Other Toy Models \dotfill \pageref{app:other_toy_models}
            \begin{enumerate}[label=\Alph{enumi}.\arabic{enumii}.\arabic*, leftmargin=0.27cm, labelsep=5pt]
                \item \textit{Model for numerically calculating the spectra} \dotfill \pageref{app:spectra_model}
                \item \textit{Learned matrix distance} \dotfill \pageref{app:learned_dist_model}
            \end{enumerate}
        \end{enumerate}
    
        \item \textbf{Theoretical Derivations and Numerical Analysis} \dotfill \pageref{app:theory_main}
        \begin{enumerate}[label=\Alph{enumi}.\arabic*, leftmargin=0cm, labelsep=5pt]
            \item Derivation of the Theoretical MSE Lower Bound \dotfill \pageref{app:mse_derivation}
            \item Impact of Projector Noise \dotfill \pageref{app:projector_noise_impact}
            \item Spectral Equivalence of $k$-fold and $1$-fold Kronecker \dotfill \pageref{app:kron_kfold_2}
            \begin{enumerate}[label=\Alph{enumi}.\arabic{enumii}.\arabic*, leftmargin=0.27cm, labelsep=5pt]
                \item \textit{Proof of $\Spec(\mathbf{A}_j) \subseteq \Spec(\bigotimes \mathbf{A}_i)$}\dotfill \pageref{app:kron_kfold_2_proof}
                \item \textit{Effect of $d$ on approximation error}\dotfill \pageref{app:kron_kfold_2_app_err}
            \end{enumerate}
            \item Residual Stream Constraints \dotfill \pageref{app:res_constraints}
            \item Depth Decoupling in Random Doubly Stochastic Mixing Terms \dotfill \pageref{app:depth_decoupling}
            \item Approximate FLOP Expressions \dotfill \pageref{app:complexity_full}
            \item Proofs of Coverage, Path Connectedness and Gradient Stability \dotfill \pageref{app:proof_of_connectedness}
            \item Impact of Spectral Radius on Gradient Stability \dotfill \pageref{app:gradient_stability_from_spectral_radius}
            \item Analyzing Convergence in Sinkhorn-Knopp \dotfill \pageref{app:sk_conv}
            \item Gradient Vanishing and Crosstalk in \mHC-lite \dotfill \pageref{app:mhclite_vanishing_gradients}
        \end{enumerate}
    
        \item \textbf{Symmetry Breaking and Residual Stream Projections} \dotfill \pageref{app:symmetry_breaking}
    
        \item \textbf{Geometry of $s$-Orthostochastic Matrices} \dotfill \pageref{app:geometry}
    
        \item \textbf{The Orthostochastic--Unistochastic Gap} \dotfill \pageref{app:stochastic_gap}
    
        \item \textbf{Other Performance \& Distance Metrics} \dotfill \pageref{app:metrics_main}
        \begin{enumerate}[label=\Alph{enumi}.\arabic*, leftmargin=0cm, labelsep=5pt]
            \item The Geodesic Distance \dotfill \pageref{app:geodesic}
            \item KL Divergence \dotfill \pageref{app:kl_divergence}
            \item Time to Convergence \dotfill \pageref{app:time_conv}
        \end{enumerate}

        \item \textbf{30M Parameter Language Model} \dotfill \pageref{app:tiny_lm}
        \begin{enumerate}[label=\Alph{enumi}.\arabic*, leftmargin=0cm, labelsep=5pt]
            \item Hyperparameters \dotfill \pageref{app:tiny_lm_hyper}
            \item Loss Curves \dotfill \pageref{app:tiny_lm_training}
            \item LLM-as-a-Judge Results for GPT-4.1 Mini and Nano \dotfill \pageref{app:judges}
            \item Total Number of Parameters \dotfill \pageref{app:tiny_lm_parameters}
        \end{enumerate}

        \item \textbf{Weighted-Average of Doubly Stochastic Matrices} \dotfill \pageref{app:alt_form}
        \begin{enumerate}[label=\Alph{enumi}.\arabic*, leftmargin=0cm, labelsep=5pt]
            \item Application to Orthostochastic\&Unistochastic Matrices \dotfill \pageref{app:alt_form_weight_avg_ortho}
            \item Application to KromHC \dotfill \pageref{app:alt_form_weight_avg_krom}
        \end{enumerate}
    
    \end{enumerate}
    
    \vskip 1em
    \newpage

    \renewcommand\thefigure{A\arabic{figure}}
    \setcounter{figure}{0}
    
    \section{Sensitivity Analysis and Model Variations}
    \label{app:sensitivity_main}
    
    \subsection{Number of Streams}
    \label{app:loss_curve_vs_d}
    
    We show that our method consistently outperforms our baselines for $d=2$ to $d=10$ by probing averaged training loss curves in figure~\ref{appfig:num_res_stream}. 
    We see that as $d$
    increases, the starting loss increases monotonically with $d$ for all go-\mHC\ and KromHC instantiations, while staying constant in \mHC-lite likely.
    We attribute this to the initialization where in \mHC-lite we parameterize closer to the barycenter $\frac{1}{d}\mathbf{J}_d$ due to the vanishing gradients at
    the polytope corners as shown in appendix \ref{app:mhclite_vanishing_gradients}, which in general is closer to a random matrix sampled from within the polytope.
    
    go-\mHC\ maintains the $\approx \times10$ speedup comapred to our baselines for the same learning rate across all values of $d$. We also see that go-\mHC\ ($s\geq 2$) and \mHC-lite both converge to the theoretical optimal (dashed blue line in figure \ref{appfig:num_res_stream}) calculated in appendix \ref{app:mse_derivation}.
    
    \subsection{Sparsity}
    \label{app:sparsity}
    
    We primarily investigated the dense regime (with sparsity $S=0$ as defined in Elhage et al.~
    \cite{elhage2022superposition}) in the main text. We use this appendix to further show that
    our toy model results remain robust when modeling the input distribution $\mathcal{X}$ to be sparse. We plot sparsity values of $S\in [0, 0.5, 0.99]$ for $d\in \{2,\dots, 10\}$ in figure~\ref{appfig:sparsity}.
    
    This allows us to apply our manifold constraints to a wide variety of networks regardless of
    the sparsity regime. We believe that this indicates that this method should work regardless of our exact interpretation of
    the mechanism of embedding of features in the residual stream given our lack of a complete mechanistic model of the residual stream \cite{elhage2023basis, brown2023privilegedconvergentbasesneural}.
    
    \subsection{SGD instead of Adam}
    \label{app:convergence_MHC_lite_w_sgd}

    While Adam's adaptive moments tend to produce more efficient convergence, we test the effect of using SGD instead to verify the robustness of our findings. We postulate that this ablation is necessary due to the nontrivial difference in the gradient landscape between our method and baseline. We show that figure~\ref{appfig:convergence_MHC_lite_w_sgd} maintains similar performance as that of figure~\ref{fig:convergence_MHC_lite_w_adam} with Adam. We conclude that the choice of static vs adaptive optimizer should have minimal impact on our parametrization's robustness.
    
    \begin{figure}[ht]
        \centering
        \includegraphics[width=\linewidth]{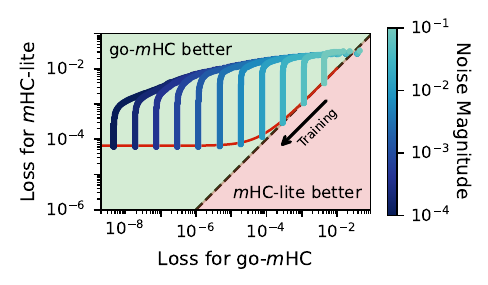}
        \caption{Ablation of figure~\ref{fig:convergence_MHC_lite_w_adam} with SGD instead of Adam. We used the same hyperparameters and only changed $\eta=0.1$. Our method remains faster at converging with a different non-momentum based optimizer.}
        \label{appfig:convergence_MHC_lite_w_sgd}
    \end{figure}
    
    \subsection{Alternative Sampling Distributions for Doubly Stochastic Matrices}
    \label{app:target_perms}

To further evaluate the robustness of our experiments, we further test the effect of the sampling distribution of target doubly stochastic matrices $\mathbf{M} \in \mathbb{R}^{d \times d}$. We implement the following four sampling techniques:

\begin{enumerate}
    \item \textbf{Sinkhorn-Knopp (SK) Normalization:} We sample an initial matrix $\mathbf{X}$ where each entry $\mathbf{X}_{ij} \sim \text{Uniform}(0, 1)$ and project the matrix using the Sinkhorn-Knopp algorithm, denoted as $\cP_\text{SK}(\mathbf{X})$. This procedure iteratively normalizes rows and columns to converge toward a doubly stochastic matrix.
    
    \item \textbf{Haar-Orthogonal Mapping:} We generate a Haar-random orthogonal matrix $\mathbf{Q} \in \mathbb{R}^{sd \times sd}$ and apply our construction from the paper to generate a random matrix. Specifically, we construct $\mathbf{M}$ as:
    \begin{equation}
        \mathbf{M}_{ij} = \frac{1}{s} \sum_{k=1}^{s} \sum_{l=1}^{s} Q_{(i-1)s+k, (j-1)s+l}^2.
    \end{equation}
    
    \item \textbf{Haar-Unitary Mapping:} Analogous to the orthogonal case, we sample a Haar-random unitary matrix $\mathbf{U} \in \mathbb{C}^{sd \times sd}$. The entries are similarly defined as $M_{ij} = \frac{1}{s} \|\mathbf{U}_{B_{i,j}}\|_F^2$, where $\mathbf{U}_{B_{i,j}}$ represents the $(i, j)$-th sub-block of size $s \times s$. This has additional phases that challenge our special orthogonal representation that is insensitive to phases.
    
    \item \textbf{Birkhoff-von Neumann Convex Combination:} Following the Birkhoff-von Neumann theorem, we can describe the set of doubly stochastic matrices as the convex hull of the set of $d!$ permutation matrices $\{P_p\}_{p=1}^{d!}$ and we can generate $M$ as:
    \begin{equation}
        M = \sum_{p=1}^{d!} \theta_p P_p,
    \end{equation}
    where the prefactors $\theta$ are sampled from a symmetric Dirichlet distribution ($\alpha=1$), representing a uniform distribution over the $(d!-1)$-simplex.
\end{enumerate}

We use method 1 in all figures in the text, unless otherwise specified.

\begin{figure}[ht]
    \centering
    \includegraphics[width=\linewidth]{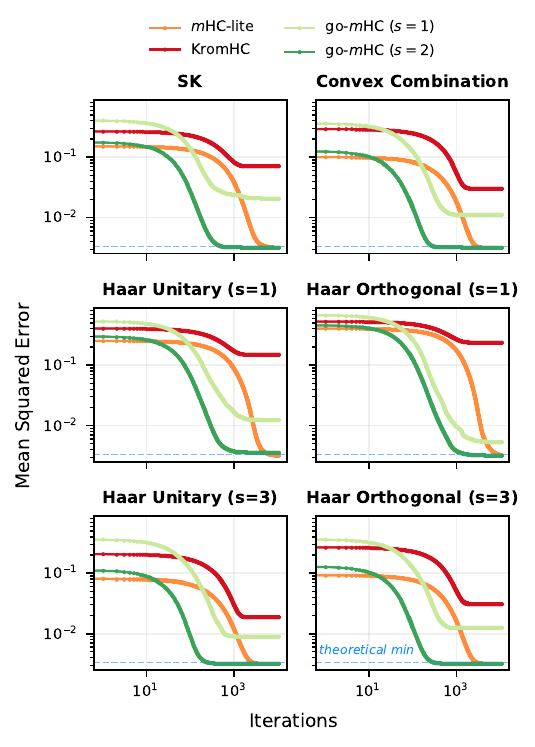}
    \caption{Loss curves for our methods and baselines for a target matrices $\{\mathbf{T_i}\}$ generated using 6 different parameterizations. We see that regardless of the distribution, our method maintains it's advantage -- in convergence and expressivity.}
    \label{fig:different_target_sampling_methods}
\end{figure}

We see in figure~\ref{fig:different_target_sampling_methods} that regardless of the method chosen to sample target doubly stochastic matrices, our method maintains the advantages discussed in the main text.

    \subsection{Other Toy Models}
    \label{app:other_toy_models}

    We discuss two substantial toy models in section~\ref{app:projector_noise_impact} and section~\ref{app:symmetry_breaking}, respectively, the model for adding projector noise and for including
    read and write terms ($\cH^{\text{pre}}$ and $\cH^{\text{post}}$) that can break stream symmetry. In this section however, we discuss two specific models, the one used in computing figures~\ref{fig:spectra_numerics}~\&~ \ref{fig:all_spectra} and an alternative formulation of our model that minimizes the distance between the learned and target matrices via the $d^2$-entrywise MSE.
    
    \subsubsection{Model for numerically calculating the spectra}
    \label{app:spectra_model}
    
    In figure~\ref{fig:all_spectra}, we aim to explore the size of the Karpelevi\v{c} region we can access using some method $\mathsf{M}$. Since we don't want to rely on $\mathsf{M}$'s initialization distribution or define a measure on the space of $\mathsf{M}$'s manifold, we opt to instead modify our toy model's objective to minimize the below loss function:
    \begin{equation*}
        \mathcal{L}_i= \min_{\lambda \in \Spec(\mathcal{H}^{\text{res}}_i)} (\lambda - e_i)^2
    \end{equation*}
    
    where $\mathcal{H}^{\text{res}}_i$ is computed and optimized using $\mathsf{M}(\cdots)$ and $e_i \in \mathbb{C}$ is a target eigenvalue sampled uniformly from the complex square $\{e_t\in \mathbb{C} \;\vert\; \left|\text{Re}(e_i)\right|\leq 1, \left|\text{Im}(e_i)\right| \leq 1 \}$.
    
    This allows us to probe the reachability of different regions in the Birkhoff polytope using a surrogate metric that post-selects on target eigenvalues. We sample $N=10^{5}$ random target eigenvalues and plot a histogram of all $\mathsf{M}$-reachable eigenvalues: $\bigcup_{i=1}^N \Spec(\mathcal{H}^{\text{res}}_i)$.
    
    \subsubsection{Learned Matrix Distance}
    \label{app:learned_dist_model}

    \begin{figure}
        \centering
        \includegraphics[width=\linewidth]{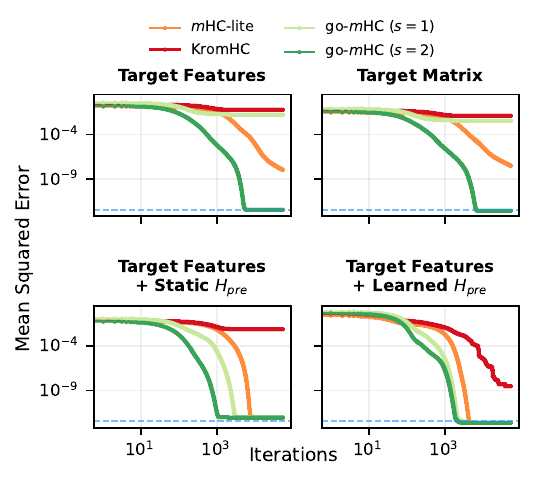}
        \caption{}
        \label{appfig:different_toy_models_methods}
    \end{figure}
    
    We also investigated using a target matrix instead of target input features to $\mathcal{F}$,
    there we see the same overall behaviour. The objective function we used there is:
    $$
    \mathcal{L}=\sum_{ij} (\mathcal{H}^{\text{res}}_{ij}-\mathbf{B}_{ij})^2
    $$
    for a target doubly stochastic matrix $\mathbf{B}$.

    We are able to draw similar conclusions as in the main text using this method as demonstrated by the sample loss curve drawn in figure~\ref{appfig:different_toy_models_methods}(b). When analyzing \mHC-lite analytically in appendix~\ref{app:mhclite_vanishing_gradients}, we use this model due to it's ease of reparameterizing using the Birkhoff–von Neumann theorem.

    \begin{figure*}[ht]
        \centering
        \includegraphics[width=0.9\linewidth]{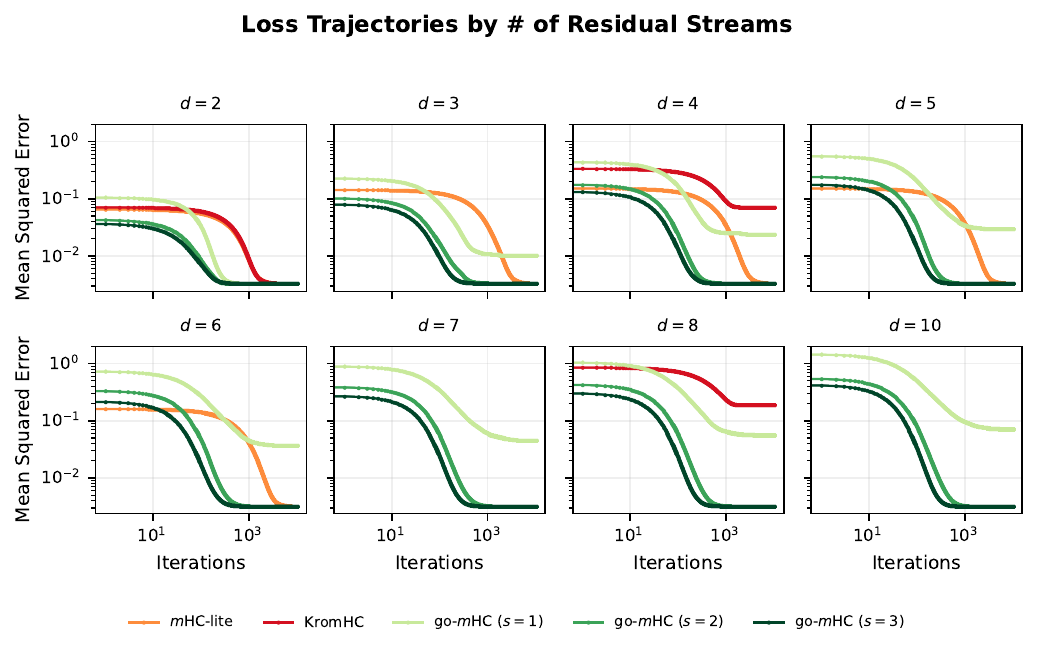}
        \caption{Convergence dynamics across the number of residual streams ($d$) and residual stream configurations ($s$). Each panel illustrates the training trajectory (MSE vs. Iterations) for dimensions $d \in [2, 10]$. We compare the baseline $m$HC-lite and KromHC against three configurations of go-$m$HC. While KromHC (red) and $m$HC-lite (orange) exhibit slower convergence and higher error floors as $d$ increases, go-$m$HC demonstrates robust scalability. Notably, increasing the number of residual streams ($s=2, 3$) in go-$m$HC accelerates convergence speed and improves final accuracy. We ran KromHC for powers of $2$ only, as this constitutes the regime where KromHC is fastest, if $i_k\neq2$, the performance interpolates to that of \mHC-lite in the limit. We don't run \mHC-lite for $d\geq 7$ due to the factorial scaling of the model.}
        \label{appfig:num_res_stream}
    \end{figure*}
    
    \begin{figure*}[p]
        \centering
        \includegraphics[width=0.8\linewidth]{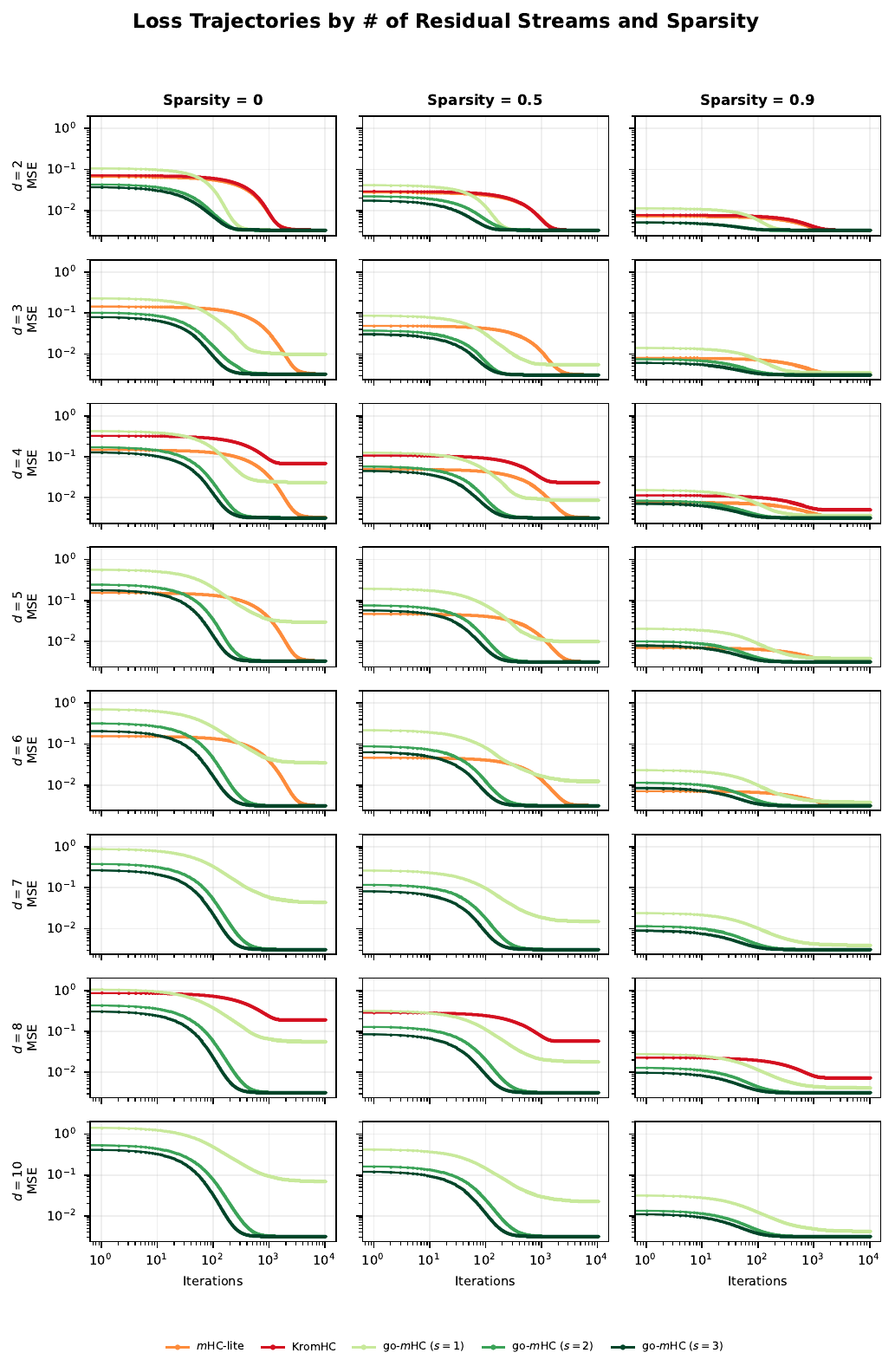}
        \caption{Similar to figure~\ref{appfig:num_res_stream}, sweeping sparsity $S\in\{0, 0.5, 0.9\}$. We see that sparsity doesn't change the overall trend, and only rescales loss due to the sparse vectors having smaller contributions.}
        \label{appfig:sparsity}
    \end{figure*}
    
    \newpage
    \section{Theoretical Derivations and Numerical Analysis}
    \label{app:theory_main}
    \subsection{Derivation of the Theoretical MSE Lower Bound}
    \label{app:mse_derivation}
    In this section, we derive the minimum achievable MSE when learning a doubly stochastic matrix $\textbf{W}$ to approximate a target $\textbf{T}$. Let $\epsilon \in \mathbb{R}_{>0}$ denote the noise magnitude, and let $\boldsymbol{\xi} \in \mathbb{R}^d$ be a random noise vector sampled such that each component $\xi_i \sim \mathcal{U}(0, \epsilon)$.
    
    The observed target is given by $\textbf{y} = \textbf{T}\textbf{v} + \boldsymbol{\xi}$. We define the loss function as the expected mean squared error over the data distribution:
    \begin{equation}
        L(W) = \mathbb{E}_{\textbf{v}, \boldsymbol{\xi}} \left[ \frac{1}{d} \| \textbf{T}\textbf{v} + \boldsymbol{\xi} - \textbf{W}\textbf{v} \|_2^2 \right]
    \end{equation}
    
    Assuming the model successfully converges to the target transformation ($\mathbf{W} = \mathbf{T}$), the structural error term vanishes. The residual loss is then dictated by the second raw moment of the noise distribution:
    \begin{align}
        L_{min} &= \frac{1}{d} \mathbb{E}_{\boldsymbol{\xi}} [ \| \boldsymbol{\xi} \|_2^2 ] \\
                &= \frac{1}{d} \sum_{i=1}^d \mathbb{E}[\xi_i^2]
    \end{align}
    
    For a variable $X$ distributed uniformly on $[0, \epsilon]$, the second moment is $\mathbb{E}[X^2] = \frac{\epsilon^2}{3}$. Substituting this into the summation:
    \begin{equation}
        \mathcal{L}_{\text{min}} = \frac{1}{d} \sum_{i=1}^d \frac{\epsilon^2}{3} = \frac{\epsilon^2}{3}
    \end{equation}
    
    This result demonstrates that the minimum MSE is independent of the dimension $d$ and is solely a function of the noise magnitude $\epsilon$. For our experimental setup where $\epsilon = 0.1$, the theoretical floor is $\mathcal{L}_{\text{min}} \approx 3.33 \times 10^{-3}$, which aligns with our empirical findings in figures \ref{fig:convergence_comparison_noise_dependence}, \ref{fig:loss_trajectory}, \ref{fig:convergence_MHC_lite_w_adam}, and \ref{appfig:num_res_stream}.

    \begin{figure}[t]
        \includegraphics[width=0.9\linewidth]{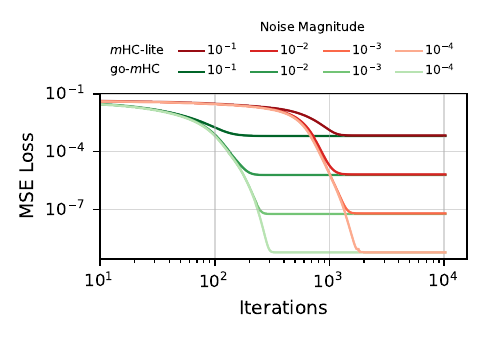}
        \caption{Averaged loss trajectories for \mHC-lite and go-\mHC\ for our toy model with different noise levels. They both achieve the same final loss, with go-\mHC consistently reaching that value $10\times$ faster for a fixed learning rate $\eta=0.001$ across noise levels.}
        \label{fig:convergence_comparison_noise_dependence}
    \end{figure}
    
    \subsection{Impact of Projector Noise}
    \label{app:projector_noise_impact}
    We further consider the case where the target transformation is subject to multiplicative noise applied on the target matrix. Let the target mixing signal have gradients that point in the direction of \mbox{$\tilde{\textbf{T}} = \textbf{T} + \boldsymbol{\Xi}$}, where $\boldsymbol{\Xi} \in \mathbb{R}^{d \times d}$ represents stochastic fluctuations on the projector. The resulting observation is $\textbf{y} = \tilde{\textbf{T}}\textbf{v} + \boldsymbol{\xi}$. 
    
    This noise is non-trivial since the $\textbf{W}$ exists within the manifold $\mathsf{B}_d$, while $\textbf{T}$ has no such restriction.
    Under the assumption $\textbf{W}=\textbf{T}$ (in the limit of $\boldsymbol{\Xi}$ not taking us out of $\mathsf{B}_d$), the minimum MSE decomposes into additive and multiplicative components:
    \begin{align}
        \mathcal{L}_{\text{min}} &= \frac{1}{d} \mathbb{E} [ \| \boldsymbol{\Xi} \textbf{v} + \boldsymbol{\xi} \|_2^2 ] \\
                &= \frac{1}{d} \left( \text{Tr}(\mathbb{E}[\boldsymbol{\Xi} \textbf{v} \textbf{v}^\top \boldsymbol{\Xi}^\top]) + \mathbb{E}[\|\boldsymbol{\xi}\|_2^2] \right)
    \end{align}

    If we assume the elements of $\boldsymbol{\Xi}$ are i.i.d. with variance $\sigma_p^2$ and $\textbf{v}$ is sampled from a distribution with $\mathbb{E}[ \textbf{v}\textbf{v}^\top] = \sigma_v^2 \mathbf{I}$, the multiplicative term simplifies to $d \sigma_p^2 \sigma_v^2$. We can remove the assumption of $\textbf{W}=\textbf{T}$ by introducing a structural error term that scales with $\sigma_v^2$ (a measure of distance away from the closest point in the Birkhoff Manifold. The total MSE floor becomes:
    \begin{equation}
    \label{appeq:projector_noise_impact}
        \mathcal{L}_{\text{min}} = d \sigma_p^2 \sigma_v^2 + \frac{\epsilon^2}{3} + \frac{1}{d}\sigma_v^2 \inf_{\mathbf{W}\in \mathsf{B}_d} \| \mathbf{W}-\mathbf{T} \|^2_F
    \end{equation}

    With the last term coming from the structural error comes from the projection of the target matrix noise onto the Birkhoff polytope.
    
    The result in equation~\ref{appeq:projector_noise_impact} suggests that if projector noise is present, the MSE floor will no longer be invariant to the dimension $d$, and will instead scale linearly with the input dimensionality. This is a similar scaling to the one found in appendix~\ref{app:kron_kfold_2_app_err}.

    Experiments in toy models reveal that we maintain the same behaviour in the limits of small and large $\sigma_p \in \{10^{-6}, 10^{-4}, 10^{-3}, 10^{-2}, 10^{-1}\}$ as shown in figure~\ref{fig:sigmap_dep}.

    \begin{figure*}
        \centering
        \includegraphics[width=\linewidth]{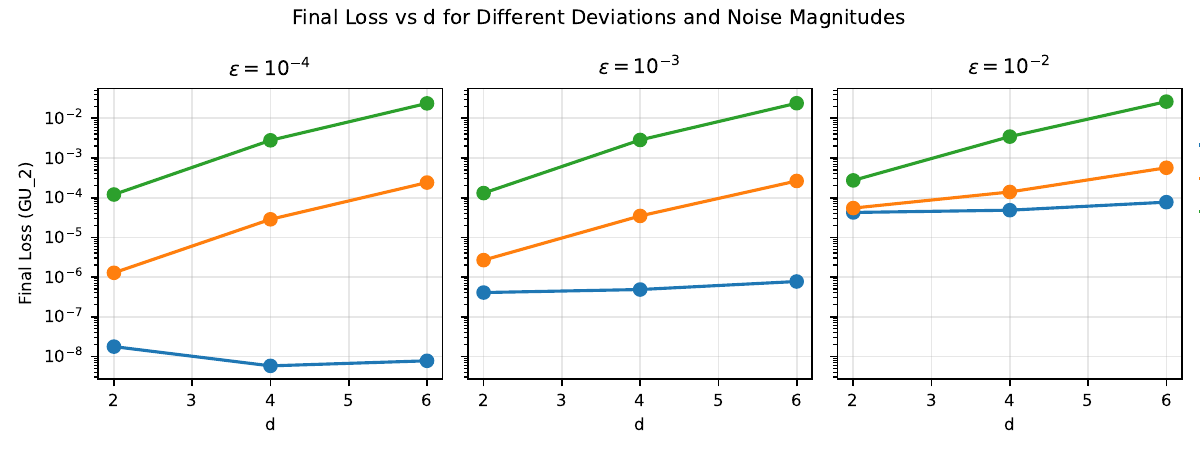}
        \caption{}
        \label{fig:sigmap_dep}
    \end{figure*}
    
    \subsection{Spectral Equivalence of $k$-fold and $1$-fold Kronecker Products}

    We show in this section that the mixing terms expressed by
    $k$-fold Kronecker Products of doubly stochastic matrices sampled from $\mathsf{B}_q$ cannot model dynamics outside of $\mathsf{B}_q$. We show a proof of this claim below and demonstrate emperically that the loss scales as $\cO(d)$ as $d>i_k$.
    
    \label{app:kron_kfold_2}
    \subsubsection{Proof of $\Spec(\mathbf{A}_j) \subseteq \Spec(\bigotimes \mathbf{A}_i)$}
    \label{app:kron_kfold_2_proof}
    
    \begin{proposition}
    The spectral space of a $k$-fold Kronecker product $\bigotimes_{i=1}^k \mathbf{A}_i$, where each $\mathbf{A}_i \in \mathsf{B}_q$, is identically the spectral space of $\mathsf{B}_q$.
    \end{proposition}
    
    \begin{proof}
    Let $\Spec(\mathbf{A}_i)$ be the spectrum of $\mathbf{A}_i \in \mathsf{B}_q$. The spectrum of the $k$-fold Kronecker product is given by the set of all possible $k$-way products:
    \begin{equation}
        \Spec\left( \bigotimes_{i=1}^k \mathbf{A}_i \right) = \left\{ \prod_{i=1}^k \lambda_i : \lambda_i \in \Spec(\mathbf{A}_i) \right\}
    \end{equation}
    Since $1 \in \Spec(\mathbf{A}_i)$ for all $i$, we can choose $\lambda_j \in \Spec(\mathbf{A}_j)$ and set all other $\lambda_{i \neq j} = 1$, which recovers the original spectrum $\Spec(\mathbf{A}_j) \subseteq \Spec(\bigotimes \mathbf{A}_i)$. 
    
    Conversely, for any doubly stochastic matrix, the eigenvalues lie within the Karpelevi\v{c} region $\mathcal{K}_q$. Because $\mathcal{K}_q$ is star-shaped with respect to the origin and closed under multiplication for elements on the unit circle, the products $\prod \lambda_i$ do not expand the spectral boundary beyond $\mathcal{K}_q$ (all $\lambda_i\leq1$). Thus, the span of the spectra remains invariant under the $k$-fold Kronecker transformation.
    \end{proof}

    \subsubsection{Effect of $d$ on approximation error in KromHC}
    \label{app:kron_kfold_2_app_err}

    We show in figure figure~\ref{appfig:kromhc_error_vs_d} that regardless of choice of $i_k$, the MSE approximation error incurred by KromHC scales with $\cO(d)$. We also show that choosing higher $i_k$ doesn't necessarily lead to a better MSE when $i_k<d$ as shown in the examples of $i_k=2$ and $i_k=4$ for $d=16$.

    We note that this scaling is similar to that found in appendix~\ref{app:projector_noise_impact}. This is in part due to being able to rewrite the deviation of our target matrix in $\mathbf{B}_d$ from it's optimal Kronecker factorized counterpart in $\{ \bigotimes_i B_i \vert B_i \in \mathbf{B}_q\}$ as noise on the target matrix $\mathbf{\Xi}$, with large variance -- this would also explain the same MSE scaling regardless of $i_k$ (independent of the $q$ in $\mathsf{B}_q$).

    \begin{figure}[t]
        \centering
        \includegraphics[width=\linewidth]{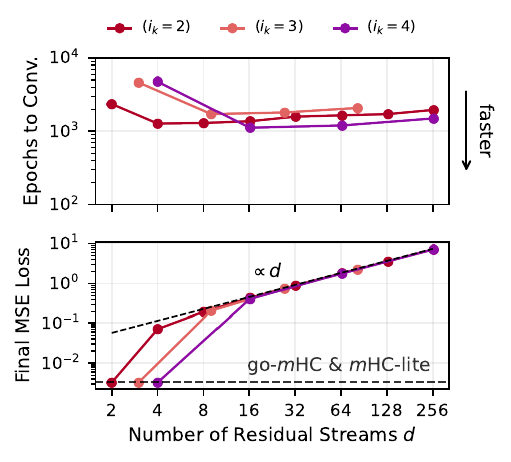}
        \caption{The minimum achievable error and convergence properties of KromHC as a function of $d$ and $i_k$. We find that the dynamics and results are mostly independent of choice of $i_k$. Namely, the approximation error of a matrix in $\mathsf{B}_d$ from a (product of) matrices in $\mathsf{B}_q$ with $q<d$ will have an MSE that scales with $d$, with little-to-no dependence on $q$ ($i_k$).}
        \label{appfig:kromhc_error_vs_d}
    \end{figure}
    
    \subsection{Residual Stream Constraints}
    \label{app:res_constraints}
    
    Unrolling equation \eqref{eq:hc} between two layers $l$ to $L$ yields:
    \begin{align}
    \medmath{
      \vx_L}&\medmath{= \underbrace{\left(\prod_{i=1}^{L-l} \cH_{L-i}^{\text{res}}\right)}_{\cH_{l\to L}^{\text{res}}} \vx_l} \nonumber\\
      &\medmath{+ \sum_{i=l}^{L-1} \left(\prod_{j=1}^{L-1-i} \cH_{L-j}^{\text{res}}\right) (\cH_i^{\text{post}})^\top\, \cF(\cH_i^{\text{pre}}\,\vx_i, \mathcal{W}_i).
      }
      \label{eq:hc_unrolled}
    \end{align}
    
    The shortcut product $\cH_{l\to L}^{\text{res}}$ controls the effective gain of the residual path. The \emph{Amax Gain Magnitude} (AGM) \cite{xie2025mhc} quantifies worst-case amplification:
    \begin{align}
      \text{AGM}^{\text{fwd}}(\cH_{l\to L}^{\text{res}}) = \max_{i} \sum_j |(\cH_{l\to L}^{\text{res}})_{ij}|,\\
      \text{AGM}^{\text{bwd}}(\cH_{l\to L}^{\text{res}}) = \max_{j} \sum_i |(\cH_{l\to L}^{\text{res}})_{ij}|.
    \end{align}
    This metric is the proxy metric used in \cite{xie2025mhc} as a surrogate for stability \& double-stochasticity.
    
    In this paper, we chose to not rely on this metric due to the exact nature of our
    algorithm:
    the definition of doubly stochastic matrices implies $\max_{i} \sum_j |(\cH_{l}^{\text{res}})_{ij}|=\max_{j} \sum_i |(\cH_{l}^{\text{res}})_{ij}|=1$. Moreover, since $\bB_d$ is closed under matrix multiplication, the product $\cH_{l\to L}^{\text{res}}$ remains doubly stochastic regardless of depth ($L-l$). 
    Learning matrices from the Birkhoff polytope $\mathsf{B}_d$ is of interest here as it provides a
    closed structure with a guarantee that the shortcut path in the residual
    stream will have a constant AGM, and therefore, avoid vanishing and exploding gradients similar to the identity mapping in ResNets.
    
    Furthermore, in the limit of $L \to \infty$, if the
    sum $\sum_{i=1}^{L-l} 1-\lambda_2(\cH_{L-i}^{\text{res}}) = \infty$, then the product $\prod_{i=1}^{L-l} \cH_{L-i}^{\text{res}}$ will
    tend to the barycenter $\frac{1}{d}\mathbf{J}_d$, where $\mathbf{J}_d$ is the all-ones matrix \cite{schwarz1980infinite}. This limit can be viewed as the identity for a repetition code on the mean-centered stream representation -- on average, far away layers have little influence on which stream information is encoded into. We show this emperically in appendix~\ref{app:depth_decoupling} and discuss potential implications of this.
    
    \subsection{Depth Decoupling in Random Doubly Stochastic Mixing Terms}
    \label{app:depth_decoupling}
    
    Another, to our knowledge, unstated advantage of using doubly stochastic matrices
    is the decoupling of $\cH^{\text{res}}_l$ from $\cH^{\text{res}}_{L}$
    when $L\gg l$. We show that numerically in figure~\ref{fig:decoupling_lambda}.
    
    \begin{figure}[t]
        \centering
        \includegraphics[width=0.9\linewidth]{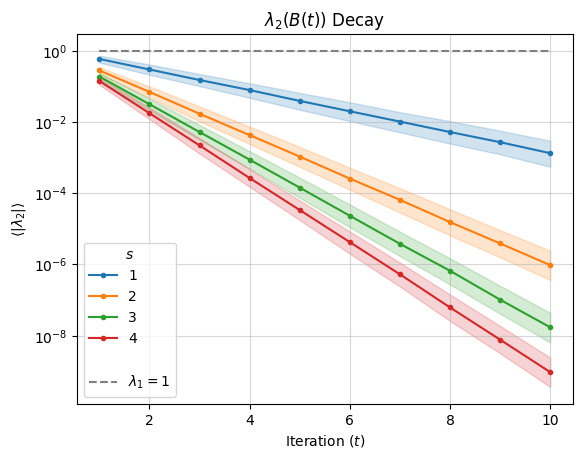}
        \caption{Random doubly stochastic matrice exhibit depth decoupling. Tracking the second eigenvalue, we see that the residual stream product tends towards the barycenter. Increasing $s$ allows for more flexibility of having less overlap of matrices. This is not surprising since we are sampling a sum of permutation matrices that are decoupled from each other.}
        \label{fig:decoupling_lambda}
    \end{figure}

    This is the same result from the previous section, where in effect, if all $\cH^{\text{res}}_i$ between $l$ and $L$ are learned independently (to some extent), the product tends to approximate the barycenter. This argument is the same as all eigenvalues $\lambda_i\neq 1$ tend to $0$, until we get to $\frac{1}{d}\mathbf{J}_d$. We think this implies that a model can simultaneously learn mixing matrices without having backward dependencies on what is in the stream, greatly simplifying the gradient lightcones (necessary gradient pathways). This would be interesting to investigate for parallelizing training and to better understand the mechanisms exhibited by structures that tend to the mean within a network.
    
    \subsection{Approximate FLOP Expressions}
    \label{app:complexity_full}
    
    We provide our explicit FLOP estimates in table~\ref{apptab:flops}. We used this to generate the asymptotics in table~\ref{tab:flop_complexity}.
    
    \begin{table}[!ht]
        \centering
        \caption{}
        \label{apptab:flops}
        \begin{tabular}{ll}
            \toprule
            \textbf{Method} & \textbf{Approximate FLOPs} \\
            \midrule
            \mHC            & $4 C S d^2$ \\
            go-\mHC\ ($s$)   & $\frac{8}{3} C (ds)^3 + 4C(ds)^2$ \\
            \mHC-lite       & $3C(d!) + 2Cd^2(d!)$ \\
            KromHC         & $3CK(i_k!) + 2CKi_k^2(i_k!) + C \left( \frac{i_k^{2K+2} - i_k^4}{i_k^2 - 1} \right)$ \\
            \bottomrule
        \end{tabular}
    \end{table}
    
    \subsection{Proofs of Coverage, Path Connectedness and Gradient Stability}
    \label{app:proof_of_connectedness}

The learnability of our proposed mapping is fundamentally governed by the geometry of the underlying loss landscape and the stability of the associated gradient signals. By employing parameterizations that ensure a smooth, well-conditioned objective function $\mathcal{L}(\theta)$, the optimization manifold avoids the sharp curvatures and isolated local minima that typically sequester gradient descent. This structural smoothness, characterized by a bounded Hessian $\nabla^2 \mathcal{L}(\theta)$ and a surjective mapping $\mathcal{P_\mathsf{M}}$ to the target matrix space, mitigates the vanishing gradient problem by maintaining a non-vanishing Jacobian $\partial \mathcal{P_\mathsf{M}} / \partial \theta$. We observe that our method offers a better parameterization compared to \mHC-lite within the polytope from toy models. In this appendix, we attempt to prove the coverage, path-connectedness, and gradient stability of our approach. We provide a formal sketch of one mechanism we encountered for vanishing gradients in \mHC-lite and a characterization of the associated cross-talk terms in appendix~\ref{app:mhclite_vanishing_gradients}.
    
    \paragraph{Analytical Gradients.} 
    We expect this space to be smoothly interpolatable within the polytope, i.e. the learned parameters $\mathbf{X}$ have a well defined gradient to smoothly learn $\mathbf{B}=\phi_{d,s}\left((\mathbf{I}-\mathbf{X})(\mathbf{I}+\mathbf{X})^{-1}\right)$.
    Given a loss $\mathcal{L}$, the gradient with respect to the orthogonal matrix $\mathbf{Q}$ follows:
    $$\mathbf{G}_Q = \frac{2}{s} \left(\frac{\partial \mathcal{L}}{\partial \mathbf{B}} \otimes \mathbf{1}_{s \times s}\right) \odot \mathbf{Q}$$
    
    Applying the chain rule through the Cayley transform $\mathcal{C}$ yields:
    \begin{equation}
        \nabla_{\mathbf{X}} \mathcal{L} = -2(\mathbf{I}+\mathbf{X})^{-\top} \mathbf{G}_Q (\mathbf{I}+\mathbf{X})^{-\top}
    \end{equation}
    The parameter gradient $\nabla_\theta \mathcal{L}$ is then the vectorization of the skew-symmetric projection $\text{proj}_{\mathfrak{so}(n)}(\nabla_{\mathbf{X}} \mathcal{L})$, where $\theta \in \mathbb{R}^{ds(ds-1)/2}$ parameterizes the skew-symmetric matrix $\mathbf{X} \in \mathfrak{so}(n)$ via the map $\mathcal{I}$.
    \newline
    \begin{lemma}
    The map $\cP_{\text{go}} = \Phi_{d,s} \circ \mathcal{C} \circ \mathcal{I}$ is $C^\infty$ (infinitely differentiable) over $\mathbb{R}^p$.
    \end{lemma}
    \begin{proof}
    $\mathcal{I}$ is linear, and $\Phi_{d,s}$ is a quadratic polynomial, both of which are $C^\infty$. The Cayley transform $\mathcal{C}$ is a rational matrix function. Since $\mathbf{X}$ is skew-symmetric, its eigenvalues are purely imaginary and $(\mathbf{I}+\mathbf{X})$ is strictly non-singular ($\det(\mathbf{I}+\mathbf{X}) \geq 1$). As a result, $\mathcal{C}$ has no poles in $\mathbb{R}^{ds(ds-1)/2}$. The composition of these smooth maps ensures that for any $\theta_0$, there exists a neighborhood $U \subset \mathbb{R}^{ds(ds-1)/2}$ such that the image $\cP_{\text{go}}(U)$ is a continuous deformation in $\mathsf{B}_d$. This guarantees that gradient-based updates allow for stable, infinitesimal ``local movement'' across the bistochastic surface during training.
    \end{proof}
    
    We also know that the entire polytope is connected, and therefore we can learn any target $s$-orthostocastic matrix.
    \begin{proposition}
    The reachable set $\mathcal{R} = \{ \cP_{\text{go}}(\theta) \mid \theta \in \mathbb{R}^{ds(ds-1)/2} \}$ is path-connected and contains the barycenter $\frac{1}{d}\mathbf{J}_d \in \mathsf{B}_d$.
    \end{proposition}
    \begin{proof}
    The parameter space $\mathbb{R}^{ds(ds-1)/2}$ is a convex topological space. As shown previously, $\det(\mathbf{I}+\mathbf{X}) \neq 0$ for all skew-symmetric $\mathbf{X}$, ensuring $\cP_{\text{go}}$ is continuous on the entire domain. By the preservation of connectedness under continuous maps, $\mathcal{R}$ is path-connected. 
    
    To show $\frac{1}{d}\mathbf{J}_d \in \text{cl}(\mathcal{R})$, let $\mathbf{Q} \in SO(ds)$ be a matrix with uniform entries $|\mathbf{Q}_{uv}| = (ds)^{-1/2}$. This yields $\|\mathbf{Q}_{block_{ij}}\|_F^2 = s/d$ for every $s\times s $ blocks, and thus $\mathbf{B}_{ij} = 1/d$. Since the image of the Cayley transform is dense in $SO(n)$, there exists a sequence $\theta_k \to \theta^*$ such that $\Phi(\theta_k) \to \frac{1}{d}\mathbf{J}_d$. Thus, the optimization can reach the maximally entropic state from any initial configuration via a continuous path.
    \end{proof}

    \subsection{Impact of Spectral Radius on Gradient Stability}
    \label{app:gradient_stability_from_spectral_radius}
    We show in this section a connection between the spectral radius of our learned matrices $\rho(\cH^{\text{res}}_l)$ and  exploding or vanishing gradients during training.
    
    Consider a deep linear model with $n$ layers and error signal $\delta_n = \frac{\partial \mathcal{L}}{\partial \hat{y}}$. The gradient of the loss $\mathcal{L}$ with respect to an early weight matrix $\mathbf{H}^{\text{res}}_1$ is given by the chain rule as $\nabla_{\mathbf{H}^{\text{res}}_1} \mathcal{L} = \left( \prod_{i=n}^{2} ({\mathbf{H}^{\text{res}}_i})^\top \right) \delta_n x^\top$. For the homogeneous case where $\mathbf{H}^{\text{res}}_i = \mathbf{H}^{\text{res}}$, the magnitude of the gradient is governed by the power of the spectral radius $\rho(\mathbf{H}^{\text{res}}) = \max\; \Spec(\mathbf{H}^{\text{res}})$:
    \begin{equation}
    \left \|\nabla_{\mathbf{H}^{\text{res}}_1} \mathcal{L}\right\| \approx \left \|\left(\left(\mathbf{H}^{\text{res}}\right)^\top\right)^{n-1} \delta_n x^\top\right\| \leq C \cdot \rho\left(\mathbf{H}^{\text{res}}\right)^{n-1}
    \end{equation}
    where $C = \|\delta_n x^\top\|$. As the depth $n \to \infty$, the term $\rho(\mathbf{H}^{\text{res}})^{n-1}$ dictates the asymptotic behavior: if $\rho(W) > 1$, the gradient norm grows exponentially ($\|\nabla_{\mathbf{H}^{\text{res}}_1} \mathcal{L}\| \to \infty$), causing exploding gradients. Conversely, if $\rho(\mathbf{H}^{\text{res}}) < 1$, the gradient norm decays exponentially ($\|\nabla_{\mathbf{H}_1^{\text{res}}} \mathcal{L}\| \to 0$), resulting in vanishing gradients. 

    \subsection{Analyzing Convergence in Sinkhorn-Knopp}
    \label{app:sk_conv}

    We point to the absence of the softmax as one of the primary mechanisms our method achieves faster convergence in. We can test a similar mechanism by ablating the $\log$ mapping in $\cP_{\text{SK}}$ to make a linear version, where the algorithm is:
    \begin{algorithm}[tb]
       \caption{Sinkhorn Linear Forward Pass}
       \label{alg:sinkhorn}
    \begin{algorithmic}
       \STATE {\bfseries Input:} Weights $W \in \mathbb{R}^{d \times d}$, iterations $K$, $\epsilon$
       \STATE $M \leftarrow |W|$
       \FOR{$k=1$ {\bfseries to} $K$}
       \STATE $M_{ij} \leftarrow \frac{M_{ij}}{\sum_{k=1}^d M_{ik} + \epsilon}$ \COMMENT{Row Normalization}
       \STATE $M_{ij} \leftarrow \frac{M_{ij}}{\sum_{k=1}^d M_{kj} + \epsilon}$ \COMMENT{Column Normalization}
       \ENDFOR
       \STATE {\bfseries return} $M$
    \end{algorithmic}
    \end{algorithm}

    \begin{figure}
        \centering
    \includegraphics[width=\linewidth]{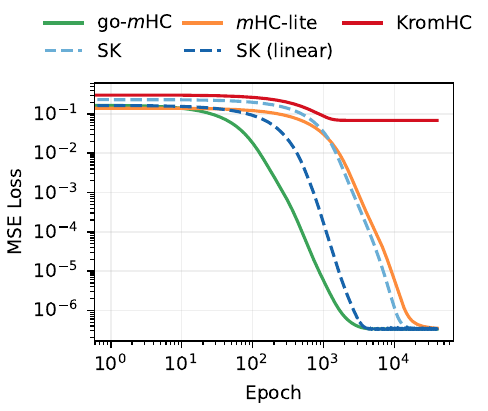}
        \caption{We additionally show loss curves that include \mHC and another variant of \mHC that transforms matrices linearly. We argue that the linear transformation is preferable in the context of our toy model due to faster convergence and the absence of vanishing gradients. We note that the speed up we see is similar to that of go-\mHC over the other methods.}
        \label{fig:loss_comparison_with_SK}
    \end{figure}

    We show in figure~\ref{fig:loss_comparison_with_SK} that the convergence rate of \mHC\ with $\cP_{\text{SK linear}}$ is faster than that of $\cP_{\text{SK}}$, converging around the same time as go-\mHC. This speed-up is similar in order to that of \mHC-lite to go-\mHC. If we add a redundant softmax normalization to go-\mHC, this also slows it down -- verifying that the softmax is a choke for gradients when optimizing paths in the Birkhoff polytope.
    
    We suspect that go-\mHC\ starts out faster than even $\cP_{\text{SK linear}}$ due to the parameterization as a rotation matrix, which exists on a smooth manifold with large proportionate action. In Sinkhorn-Knopps, the path to a target matrix ``zig-zags'' due to the iterative procedure of updating (which usually forces you out of the manifold) followed by a projection down onto the manifold-of-interest.
    
    \subsection{Gradient Vanishing and Crosstalk in \mHC-lite}
    \label{app:mhclite_vanishing_gradients}
    
    Let $\vec \alpha$ be the learned $n!$ parameters for constructing a $n\times n$ doubly stochastic matrix
    as in \mHC-lite. We will calculate the gradient for parameter updates assuming mean squared error loss
    from some target matrix:
    
    \begin{align}
    &L = \dfrac{1}{n^2} \sum_{1\leq i, j\leq n} (\mathbf{B}_{ij}- \mathbf{T}_{ij})^2\\
    &\mathbf{B} = \sum_{k=1}^{n!} \alpha_k \mathbf{P}_k\\
    &\frac{\partial L}{\partial \mathbf{A}} = 
    \end{align}
    
    Analytically solving for the gradients in \mHC-lite when optimizing to reach a final permutation matrix $\mathbf{P}_{\text{target}}$, we get that
    
    \begin{equation}
    \frac{\partial L}{\partial z_{\text{target}}} = \alpha_{\text{target}} \left( \frac{\partial L}{\partial \alpha_{\text{target}}} - \sum_j \alpha_j \frac{\partial L}{\partial \alpha_j} \right)
    \end{equation}
    
    Where $\alpha = \text{softmax}(\mathbf{z})$, the subscript target corresponds to the entry at the index corresponding to $\mathbf{P}_{\text{target}}$, and the loss gradient wrt. $\alpha$ is:
    
    \begin{align}
    \frac{\partial L}{\partial \alpha_i} &= \frac{2}{n^2} \left\langle \sum \alpha_k \mathbf{P}_k - (\mathbf{P}_{\text{target}} + \mathbf{E}), \mathbf{P}_i \right\rangle\\
    &= \frac{2}{n^2} \Biggl( \;\;
        \underbrace{\alpha_i \langle \mathbf{P}_i, \mathbf{P}_i \rangle}_{\text{Self-Correction}} 
        - \underbrace{\langle \mathbf{P}_{\text{target}}, \mathbf{P}_i \rangle}_{\text{Signal}} \nonumber\\
        &\;\;\;\;\;\;\;\;\;\;\;\; + \underbrace{\sum_{k \neq i} \alpha_k \langle \mathbf{P}_k, \mathbf{P}_i \rangle}_{\text{Crosstalk}} - 
        \underbrace{\langle \mathbf{E}, \mathbf{P}_i \rangle}_{\text{Noise Floor}} 
    \;\;\Biggl)
    \end{align}

    Note that the gradient update for $z_{\text{target}}$ is gated by the factor $\alpha_{\text{target}}$, which creates an inherent tension in early optimization: when $\alpha_{\text{target}}$ is initialized near $\frac{1}{n!}$ (i.e., the softmax is approximately uniform), this prefactor is small, suppressing the effective gradient signal and slowing initial convergence toward $\mathbf{P}_{\text{target}}$. This is a structural feature of the softmax parameterization---the model cannot ``commit'' quickly to a target permutation until it has already begun to assign probability mass to it, yielding a self-reinforcing but initially sluggish convergence regime.

The crosstalk term $\sum_{k \neq i} \alpha_k \langle \mathbf{P}_k, \mathbf{P}_i \rangle$ captures interference from competing permutations. Since all permutation matrices satisfy $\langle \mathbf{P}_k, \mathbf{P}_i \rangle \in \{0, 1, \ldots, n\}$ depending on the number of shared assignments, permutations that agree with $\mathbf{P}_i$ on many positions contribute disproportionately to this term. In the regime where $\alpha$ is near-uniform, the crosstalk sum approximates $\langle \mathbf{B}, \mathbf{P}_i \rangle$ with $\mathbf{B} \approx \frac{1}{n}\mathbf{J}$ (where $\mathbf{J}$ is the all-ones matrix), corresponding to the Birkhoff--von Neumann barycenter of the permutation polytope. This provides a geometric interpretation: early in training, the gradient updates push $\mathbf{B}$ away from the barycenter toward $\mathbf{P}_{\text{target}}$, but the crosstalk from near-uniform $\alpha$ keeps $\mathbf{B}$ gravitating toward the interior of the polytope. As $\alpha_{\text{target}}$ grows and mass concentrates, crosstalk diminishes and the gradient signal from the \textbf{Signal} term dominates, driving $\mathbf{B}$ toward a vertex of the polytope.

The \textbf{Noise Floor} term $\langle \mathbf{E}, \mathbf{P}_i \rangle$ contributes a constant bias to each gradient, reflecting the deviation of the target $\mathbf{T} = \mathbf{P}_{\text{target}} + \mathbf{E}$ from a pure permutation. When $\mathbf{E} = \mathbf{0}$ (i.e., the target is itself a permutation matrix), this term vanishes and the gradient structure simplifies considerably. For noisy targets, however, the noise floor introduces a persistent gradient offset that prevents exact convergence to any vertex of the polytope, instead biasing $\mathbf{B}$ toward a weighted combination of permutations that collectively minimize projection error against $\mathbf{E}$.
        
    \section{Symmetry Breaking and Residual Stream Projections}
    \label{app:symmetry_breaking}

    The residual stream in Hyper-Connections has non-trivial boundaries -- connections to computation layers ($\cF$) through projection matrices (e.g. $\cH^{\text{pre}}$). While the mixing terms have imposed constraints (in the Birkhoff Polytope for the baselines and our method, any highly imposed symmetry from factor products in KromHC), the effective mixing term into $\cF$ includes a component derived by the read and write matrices in the model. An example pathway we can come up with is that where the residual block learns to average out a feature in a trace by writing $-x$ back to the residual stream on one of the streams, causing an asymmetric write even if the manifold usually has an imposed symmetry on mixing. We can also have a pathway where from the residual blocks point of view, $\cH_i^{\text{pre}} \odot (\cH_i^{\text{res}}X_j)$ breaks the symmetry of $\cH_i^{\text{res}}$ and similarly allow for asymmetric reads.
    
    In this appendix, we investigate the effect of adding a read term to our toy model and see that it doesn't impact the claims in the main paper as shown in figure~\ref{appfig:Hpre_symmetry_breaking}.

    \begin{figure}[th]
        \centering
        \includegraphics[width=\linewidth]{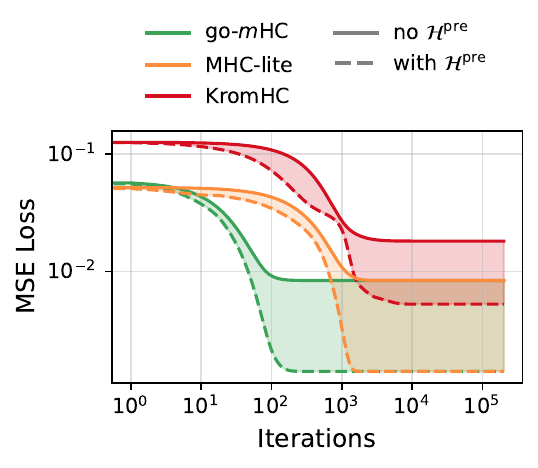}
        \caption{Sensitivity of convergence trajectories to the inclusion of a stream read term $\mathcal{H}^{\text{pre}}$ which can break KromHC's symmetry. We compare the impact of the learning signal being generated after an $\mathcal{H}^{\text{pre}}$ projection (dashed lines) against the signal originating from within the residual stream only (solid lines). While the inclusion of a learned $\mathcal{H}^{\text{pre}}$ allows for symmetry breaking in KromHC, it lowers the overall MSE floor (shaded regions) for all methods due to it's ability to sample outside the set of stochastic matrices -- which seems to be the dominant effect here. The relative performance hierarchy across all models remains invariant to the choice of toy model. In both configurations, go-\mHC (green) demonstrates the fastest convergence and lowest final loss, whereas KromHC (red) consistently exhibits the highest error and, along with \mHC-lite, the slowest training dynamics. This suggests that the advantages of the our proposed method are independent of the read/write symmetry-breaking in Hyper Connections.}
        \label{appfig:Hpre_symmetry_breaking}
    \end{figure}
    
    We introduce a tensor contraction into our model that implements a ``circuit''
    involving the $\cH^{\text{pre}}$ term. While $\cH^{\text{res}}$
    calculated via Kronecker products preserves symmetry, $\cH^{\text{pre}}$ is allowed to break this symmetry. We now assume the best case scenario of static
    mappings where we learn input-independent $\cH^{\text{pre}}$ and $\cH^{\text{res}}$ with the objective:
    \begin{equation}
        \mathcal{L}_i=\sum_j|| \cH_i^{\text{pre}} \odot (\cH_i^{\text{res}}X_j) - \mathbf{P}_i \odot (\mathbf{T}_iX_j+ \vec \epsilon)||^2_2
    \end{equation}
    where $\odot$ denotes a linear projection from the $d$-residual streams to $1$ stream in the MLP input basis and $\mathbf{P}_i$ denotes a random target projection matrix.
    
    When comparing the ability of $\cH^{\text{pre}}$ to break symmetry when
    reading from (or equivalently, $\cH^{\text{post}}$ writing to) the residual stream, we compare the convergence of a learned $\cH^{\text{pre}}_i$
    transformation with that of fixing $\cH^{\text{pre}}_i = \mathbf{P}_i$ to normalize the loss units.

    \section{Geometry of $s$-Orthostochastic Matrices}
    \label{app:geometry}

    In the main text, we use the spectrum as a surrogate for expressivity due to its ease of interpretability while still being able to represent all $d\times d$ $k$-cycles for $k\leq d$. The spectrum can be thought of as the ``shadow'' of the Birkhoff polytope. In figure~\ref{appfig:orthostochastic_boundary_3cycle}, we show a slice through the polytope instead of the shadow, where we look at the polygon composed of the $3$-cycle in the $d=3$ case. Namely the matrices defined by the cyclic permutations: $(), (123), (132)$ are the corners of the slice shown in figure~\ref{appfig:orthostochastic_boundary_3cycle}. The coverage in figure~\ref{appfig:orthostochastic_boundary_3cycle} for $s=1$ and $s=2$ matches that of the spectrum coverage within the Karpelevi\v{c} region. This motivates our choice of a metric that approximates expressivity despite it only probing the ``shadow'' of the polytope. We know the $s$-unistochastic set is dense, and don't expect there to be holes within it. We use the emperical results that we can usually approximate $s$-unistochastic matrices with an $s$-orthostochastic matrix to justify that the holes are few to non-existent.

To visualize the density and support of $s$-orthostochastic matrices within the Birkhoff polytope $\mathcal{B}_3$, we project \mbox{$\mathbf{B} \in \mathcal{B}_3$} onto the 2D plane defined by the cyclic group $\mathcal{C}_3 = \{I, \sigma, \sigma^2\}$. We compute coefficients $c_k = \frac{1}{3} \text{Tr}((\sigma^k)^\top B)$ and map them to $\mathbb{R}^2$ using the vertices of an equilateral triangle $\{v_k\}_{k=0}^2$:
\begin{equation}
    \mathbf{x} = \sum_{k=0}^2 c_k v_k, \quad v_k = \left( \sin\frac{2\pi k}{3}, \cos\frac{2\pi k}{3} \right)
\end{equation}
The support is estimated by sampling $N=10^8$ matrices and computing a binary occupancy grid over the projected coordinates, with boundaries extracted via contouring the resulting indicator function.

We alternatively also show multiple 3D visualiztions of the polytope using different support bases and rotations in $\mathbb{R}^3$ in figure~\ref{appfig:birkhoff_polytope_slices}. 

    \begin{figure}[ht]
        \centering
        \includegraphics[width=1\linewidth]{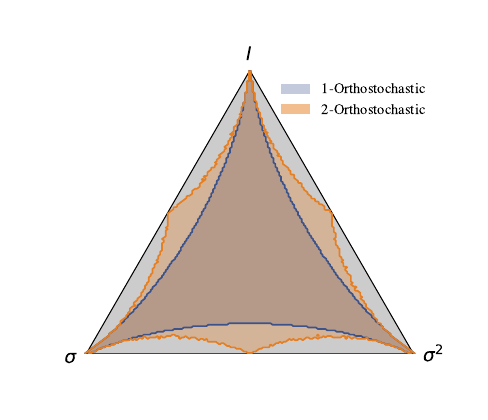}
        \caption{Numerically calculated projection on the 3-cycle $I, \sigma, \sigma^2$ with $\sigma$ being a cyclic permutation matrix ($\sigma^3=1$).}
        \label{appfig:orthostochastic_boundary_3cycle}
    \end{figure}
    
    \begin{figure}[h]
        \centering
        \includegraphics[width=\linewidth]{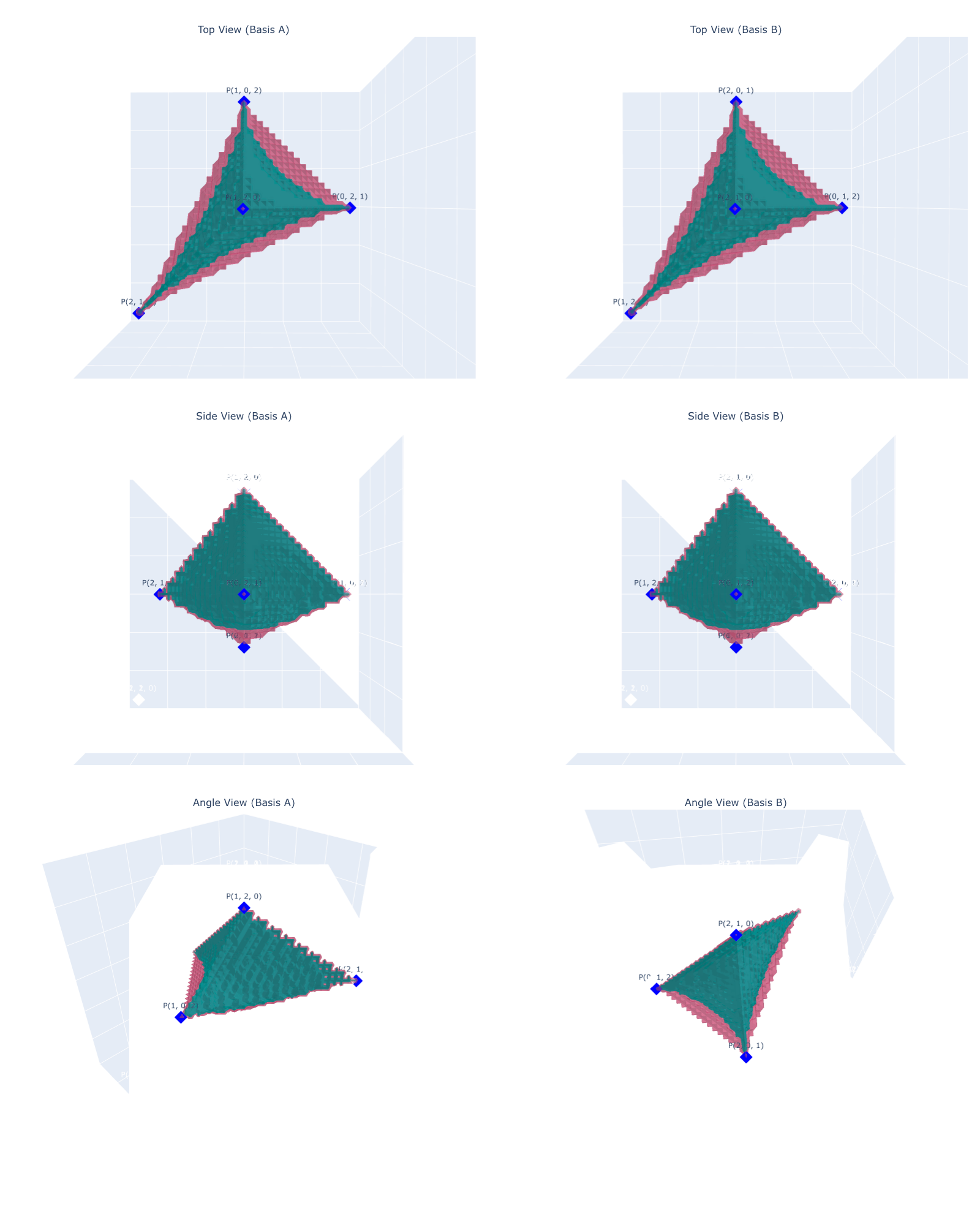}
        \caption{Visualizations of the Birkhoff Polytope $\mathsf{B}_d$. The green region denotes the support of the $1$-orthostochastic matrices and the pink region is the more general coverage of the Birkhoff polytope that isn't covered when $s=1$.}
        \label{appfig:birkhoff_polytope_slices}
    \end{figure}

    \section{The Orthostochastic--Unistochastic Gap}
    \label{app:stochastic_gap}

    While we used orthogonal matrices in go-\mHC, we in principle could have also used unitary matrices with complex entries instead. We conjecture that the implications on the geometry and expressivity are minimal when $s\geq 2$, and that the convergence is of similar order of magnitude, although it might exhibit a different curvature depending on the position and velocity vectors within the manifold.
    
    If expressivity is reduced, we expect that when sampling matrices, we get worse overall performance. We see in the 3 examples from figure~\ref{appfig:unistochastic_gap} that the total achievable loss is similar for unistochastic and doubly stochastic matrices for $d\in \{3, 6\}, s=2, \epsilon\in\{10^{-4}, 10^{-2}\}, \sigma_p \in \{0, 1\}$, regardless of the target matrix location. Plot~\ref{appfig:unistochastic_gap}(a) targets a matrix at the boundary of the polytope (specifically, the average of two permutation matrices) and Plot~\ref{appfig:unistochastic_gap}(b-c) target a random matrix sampled from $\cU(0, 1)$ with Sinkhorn-Knopp applied on it.
    
    Since the boundaries of the sets of orthostochastic and unistochastic matrices are the
    same, we additionally expect that they are both equally able to express the boundary of the polytope, except for the additional relative phases that need to cancel out in the unitary matrices. We see this manifest as slower convergence for the unitaries.

    On the other hand, for a random point within the polytope, there are many more solutions of phases
    that satisfy the target matrix condition, and therefore, we see faster convergence. We think this is due to being able to go to a point that is overall closer to the initialization due to the phase terms allowing for interference.

    \begin{figure}[ht]
    \centering
    
    \begin{subfigure}{\linewidth}
    \vspace{0.3em}
        \centering
        \textbf{(a)} \hfill \textbf{Boundary Targets, $d=3$} \hfill \hfill \phantom{(a)} 
        \includegraphics[width=\linewidth]{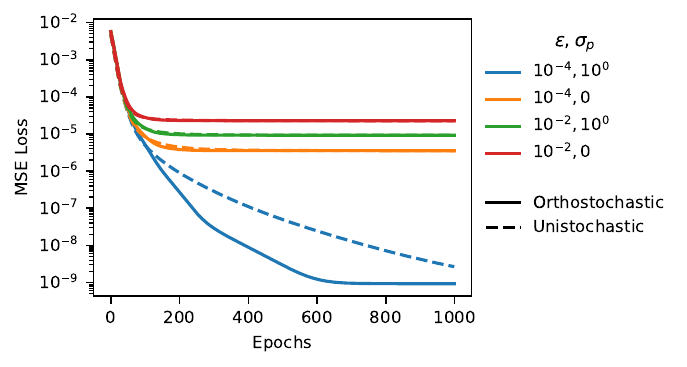}
        \vspace{-1em} 
    \end{subfigure}

    \begin{subfigure}{\linewidth}
        \centering
        \textbf{(b)} \hfill \textbf{Random Targets, $d=3$} \hfill  \hfill \phantom{(b)}
        \includegraphics[width=\linewidth]{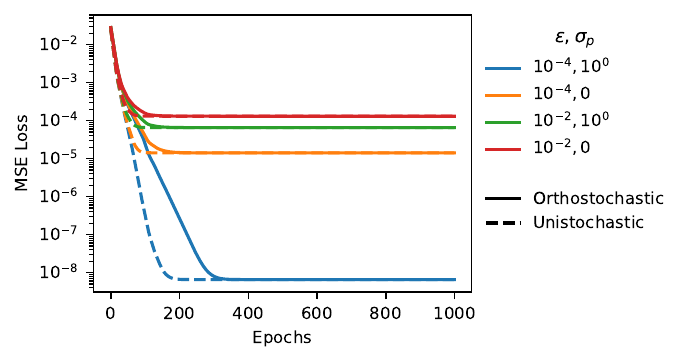}
        \vspace{-1em}
    \end{subfigure}

    \begin{subfigure}{\linewidth}
        \centering
        \textbf{(c)} \hfill \textbf{Random Targets, $d=6$}\hfill  \hfill \phantom{(c)} 
        \includegraphics[width=\linewidth]{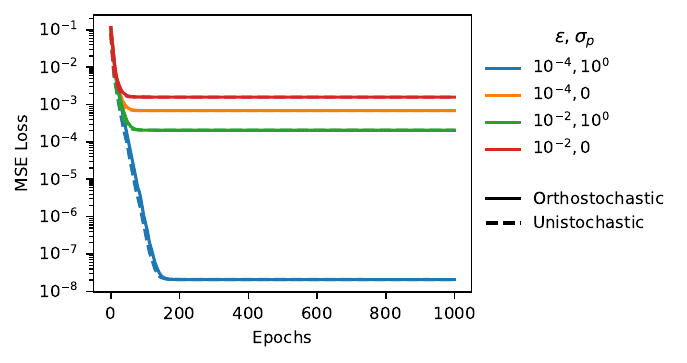}
    \end{subfigure}

    \caption{}
    \label{appfig:unistochastic_gap}
\end{figure}

    \section{Other Performance \& Distance Metrics}
    \label{app:metrics_main}
    
    \subsection{The Geodesic Distance}
    \label{app:geodesic}

    To better understand the rate of convergence of different methods, we investigate the use of a distance metric to quantify the ``progress'' and ``degree of change'' in the path taken within the polytope. Given any doubly stochastic matrix $\mathbf{W}$, we can estimate a matrix $\mathbf{R}^W$ in $SO(d)$ such that $\mathbf{W}_{ij}\approx \left(\mathbf{R}^W_{ij}\right)^2$ -- we can also calculate higher-dimensional generalizations of this (using the norm of $s\times s$ blocks to have a matrix in $SO(ds)$) to reduce the approximation error. We don't explore higher-dimensional generalization of this method here for brevity.

    We implement a version of the $s=1$ reverse map $\mathbf{W}\to \mathbf{R}^W$ as:
    \begin{equation}
        \mathbf{R}^W = \mathbf{H} \odot \mathbf{W}^{\odot \frac{1}{2}}
    \end{equation}
    with $\mathbf{H}$ a parity matrix that fixes the sign-gauge symmety and $\mathbf{W}^{\odot \frac{1}{2}}$ denoting the elementwise square root.

    This then allows us to measure the geodesic distance in radians between $\mathbf{W}$ and a target matrix $\mathbf{T}$:
    \begin{equation}
        \phi_R(\mathbf{W}, \mathbf{T}) = \dfrac{1}{\sqrt{2}} \left\| \log \left( \left( \mathbf{R}^W \right)^\top \mathbf{R}^T \right)\right \|_F
    \end{equation}

    $\phi_R(\cdots)$ is the angle between the two approximate orthogonal matrices composing our target. We plot $\phi_R(\mathbf{S}, \mathbf{W})$ and $\phi_R(\mathbf{W}, \mathbf{T})$ in figure~\ref{fig:geodesic_example}.

    \begin{figure}
        \centering
        \includegraphics[width=\linewidth]{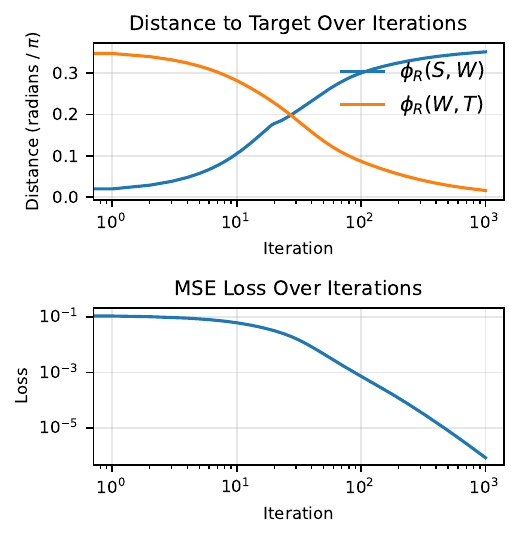}
        \caption{}
        \label{fig:geodesic_example}
    \end{figure}

    Additionally, to account for the $O(n)$ sign-gauge symmetry inherent in the unistochastic mapping $O_{ij} = M_{ij}^2$, we align the recovered matrix by performing a combinatorial sign-alignment (gauge fixing) to ensure the recovered rotation is compared to the target in the same quadrant of the Stiefel manifold.

    We don't discuss this metric in the main paper due to it's structure favoring our method, but we think this is an interesting metric to study convergence from, especially using the two angle metrics from start and to target, or a comparative angle between two methods.

    \subsection{KL Divergence}
    \label{app:kl_divergence}

    We also explore modelling matrix distance using the KL divergence due to their exact forms, instead of only relying on the Karpelevi\v{c} region and the geodesic distance (which is inherent to our method and potentially favors our formulation). Each row/column can be thought of as a discrete probability distribution and therefore, we can measure the distance between two distributions, the set of possible target distributions $\mathcal{T}$ and our method distribution $\mathcal{P}$. We show in figure \ref{fig:kl_comparison} results consistent with the rest of the paper where KromHC tends to have the worst approximation (higher KL divergence) and the approximation improves with $s$ and, emperically, it tends to be near the optimal limit (no divergence) when $s\geq 2$.

    \begin{figure}
        \centering
        \includegraphics[width=\linewidth]{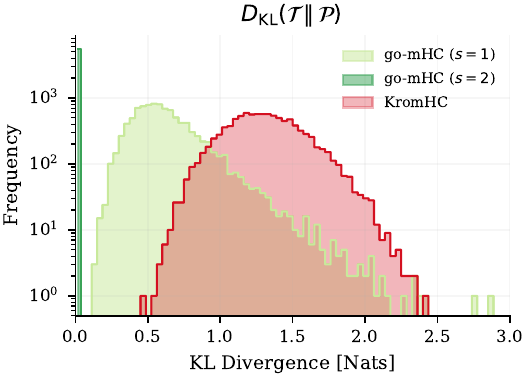}
        \caption{}
        \label{fig:kl_comparison}
    \end{figure}

    \subsection{Time to Convergence}

    Another measure we use in the main text is the epochs to convergence for each parameterization method. We also consider whether or not the method converges to the theoretical minimum achievable loss, calculated as a function of noise magnitude ($\epsilon$):

    \begin{equation}
        \mathcal{L}_{\text{min}}=\dfrac{\epsilon^2}{3}
    \end{equation}

    For most figures in the main text, we use $\epsilon=10^{-1}$, however the trend is irrespective of the value chosen.

    We calculate time to convergence as the time at which a method that reaches a $5\%$ threshold in loss away from it's final converged value. We showcase sample training curves where label the converged epoch in figure~\ref{fig:convergence_indiv_loss_trajectory}.

    In the loss curves shown throughout the paper (e.g. Figure~\ref{fig:loss_trajectory}), we can see go-\mHC\ converge at around 500 epochs for $s\in\{2,3\}$, while \mHC-lite converges significantly later (around 10,000 epochs). KromHC converges to a higher overall value. We investigate the limits of KromHC in figure~\ref{appfig:kromhc_error_vs_d}.
    
    \begin{figure}
        \centering
        \includegraphics[width=\linewidth]{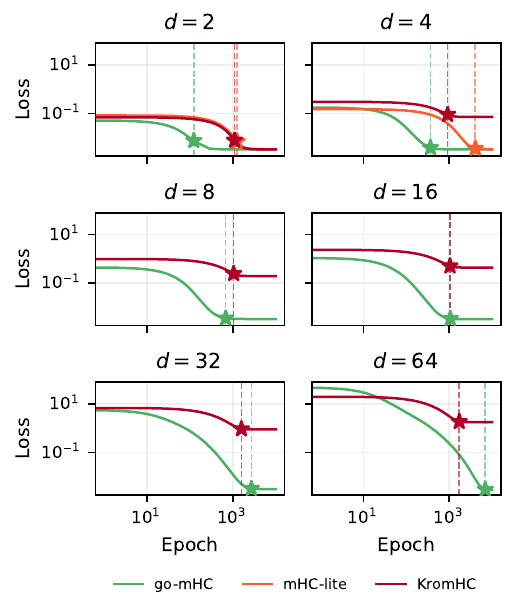}
        \caption{}
        \label{fig:convergence_indiv_loss_trajectory}
    \end{figure}
        \label{app:time_conv}

    \section{30M Parameter Language Model}
    \label{app:tiny_lm}

    \subsection{Model, Hyperparameters \& Training Details}
    \label{app:tiny_lm_hyper}

    We trained a 6-layer Transformer-based language model on the TinyStories dataset using a modified version of the nanoGPT codebase \cite{nanogpt}. 

    The model follows the decoder-only transformer architecture, specifically configured for the \texttt{TinyStories} synthetic dataset. The model consists of $6$ layers with an embedding dimension of $d_{\text{model}} = 384$. Each layer utilizes a multi-head attention mechanism with $h=6$ heads, resulting in a per-head dimension of $d_k = d_v = 64$. The position-wise feed-forward networks (FFN) employ a hidden dimensionality of $d_{ff} = 4d_{\text{model}} = 1536$.

    Our model utilizes the standard GPT-2 tokenizer with a vocabulary size of $V = 50,257$. The embedding layer, $\mathbf{W}_e \in \mathbb{R}^{V \times d_{\text{model}}}$, accounts for approximately two thirds of our parameters (20M of the 30M total parameters). We did not test smaller embedding matrices, but this would be interesting to do given the simpler vocabulary of \texttt{TinyStories}.
    
    Training was performed using the AdamW optimizer with a peak learning rate of $1 \times 10^{-3}$, decaying to $1 \times 10^{-4}$ via a cosine schedule over 20,000 iterations. To ensure training stability, we employed a linear warmup of 500 steps. We utilized an effective batch size of 128 sequences, achieved through gradient accumulation on a single NVIDIA RTX 4090 GPU using \texttt{bfloat16} mixed-precision training. Detailed hyperparameters are summarized in Table \ref{apptab:tinylm_hyperparams}.

    \begin{table}[ht]
    \centering
    \caption{Hyperparameter configurations for training on the TinyStories dataset.}
    \label{app:judges}
    \label{apptab:tinylm_hyperparams}
        \begin{tabular}{@{}llc@{}}
        \toprule
        \textbf{Category} & \textbf{Hyperparameter} & \textbf{Value} \\ \midrule
        \textbf{Model} & Layers & 6 \\
         & Attention Heads & 6 \\
         & Embedding Dimension & 384 \\
         & Context Window & 512 \\
         & Dropout & 0.1 \\ \midrule
        \textbf{Optimization} & Peak Learning Rate & $1 \times 10^{-3}$ \\
         & Minimum Learning Rate & $1 \times 10^{-4}$ \\
         & Learning Rate Schedule & Cosine Decay \\
         & Warmup Steps & 500 \\
         & Weight Decay & 0.1 \\ \midrule
        \textbf{Training} & Batch Size (Total) & 32 \\
         & Max Iterations & 20,000 \\
         & Mixed Precision & \texttt{bfloat16} \\ \bottomrule
        \end{tabular}
    \end{table}

    In our experiments, we swap out the residual stream with the various proposed hyper-connection variants, all of which were built on top of the code in Appendix J of \citet{zhu2025hyper} and the code published by \citet{yang2026mhclite}. We additionally use the implementation of KromHC from \citet{zhou2026kromhc}.

    \subsection{Loss Curves \& Training Stability}
    \label{app:tiny_lm_training}

    Figure~\ref{fig:cross_entropy_loss} shows the cross-entropy loss when training various proposed HC variants. There are no perceptable difference in the CE Loss between the different methods while training -- we suspect further experiments using smaller embedding matrices might help with any visibility issues here.

    \begin{figure}
        \centering
        \includegraphics[width=0.8\linewidth]{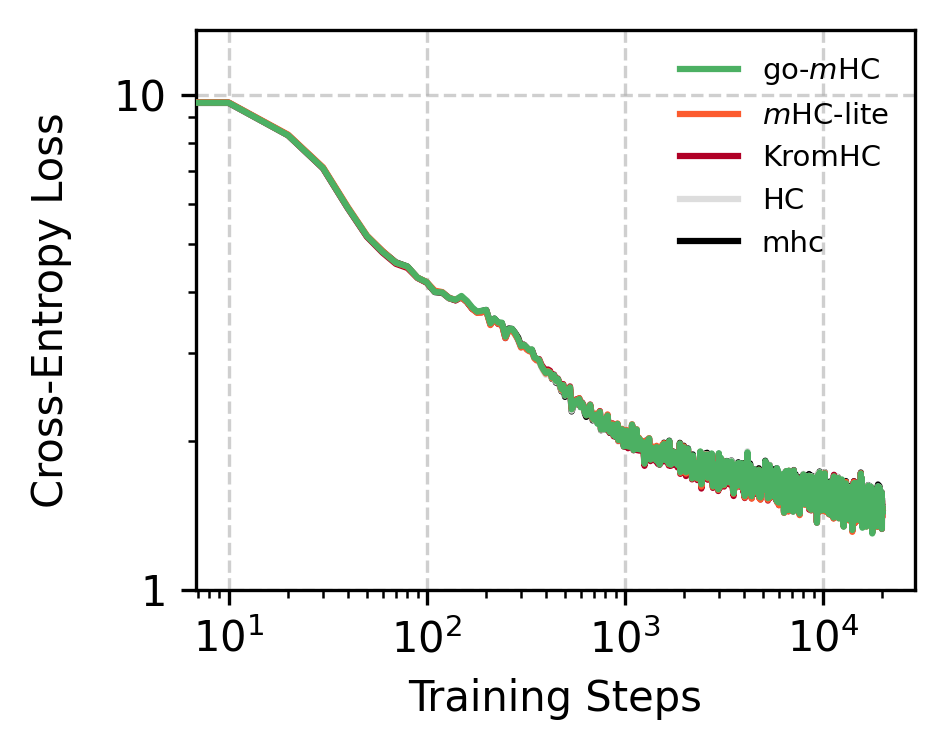}
        \caption{}
        \label{fig:cross_entropy_loss}
    \end{figure}

    We attempt to quantify training stability across architectural variants using a jitter metric $J_t$, in figure \ref{fig:cross_entropy_loss}b, as the normalized rolling standard deviation of the gradient norm $g_t = \|\nabla_{\theta} \mathcal{L}(\theta_t)\|_2$.Using a sliding window of size $w=10$, we calculate the rolling mean $\mu_t$ and standard deviation $\sigma_t$:$$\mu_t = \text{mean}(\{g_i\}_{i=t-w+1}^t) \quad , \quad \sigma_t = \text{std}(\{g_i\}_{i=t-w+1}^t)$$The final jitter is the Coefficient of Variation (CV), adjusted for numerical stability with $\epsilon = 10^{-8}$:$$J_t = \frac{\sigma_t}{\mu_t + \epsilon}$$This dimensionless ratio measures the relative volatility of gradient updates, allowing for a fair stability comparison between models with different gradient magnitudes.

    \begin{figure}[t]
        \centering
        \includegraphics[width=0.8\linewidth]{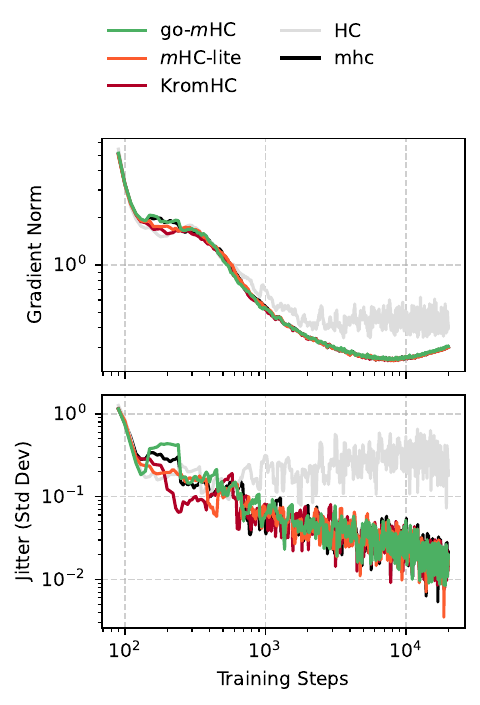}
        \caption{}
        \label{fig:gradient_norm_jitter_comparison}
    \end{figure}

    \begin{figure}
        \centering
        \includegraphics[width=0.8\linewidth]{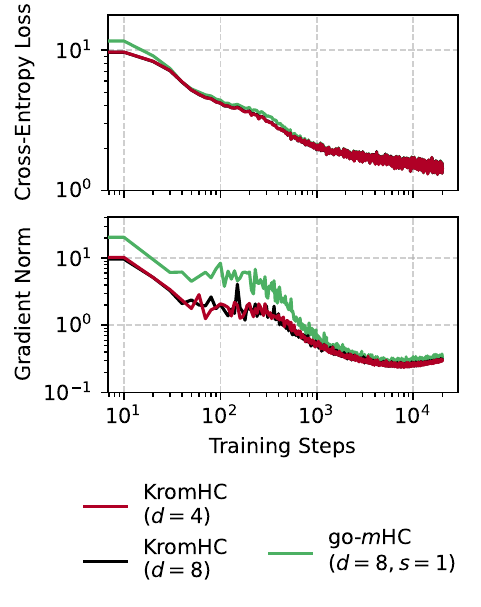}
        \caption{A version of figures~\ref{fig:cross_entropy_loss} \& \ref{fig:gradient_norm_jitter_comparison} with $d=8$ for KromHC and go-\mHC ($s=1$).}
        \label{fig:training_trace_d8}
    \end{figure}

    \subsection{LLM-as-a-Judge Results for GPT-4.1 Mini and Nano}
Tables~\ref{tab:mini_70} and \ref{tab:nano_70} provide the full performance breakdown for the smaller judge variants. Table~\ref{tab:pairwise_alignment} provides the inter-judge alignment.

\input{tables_llm_judge_appendix}

    \subsection{Total Number of Parameters}
    \label{app:tiny_lm_parameters}

    We detail the total number of parameters needed to parameterize our full language models using the various methods in table~\ref{tab:parameter_delta_grouped}. We note that for efficiency purposes, we generated skew-symmetric matrices using $d^2$ parameters instead of $\frac{d(d-1)}{2}$ parameters due to faster performance on GPUs from our tests, even for small $d$.

    Table~\ref{tab:parameter_delta_grouped} shows how to maintain using \mHC-lite at $d=8$, we need to allocate 1.4B parameters to our model that has 20M parameters for embedding and 10M parameters in the transformer blocks. This leads us to believe that \mHC-lite is not scalable with $d$. go-\mHC is consistently between \mHC-lite and KromHC, even for tiny $d$.

    \begin{table*}[h]
\centering
\caption{Model Parameter Overhead ($\Delta$) Grouped by Dimension $d$}
\label{tab:parameter_delta_grouped}
\begin{tabular}{lccc}
\hline
\textbf{Model Variant} & \textbf{Config} & \textbf{$\Delta$ Parameters (M)} & \textbf{Total Parameters (M)} \\ \hline
Residual Baseline      & ---             & 0.00                            & 30.14                         \\ \hline
{\mHC-lite}     & $d=4$ &  0.60                   & 30.74                \\
KromHC                 & $d=4$           & 0.23                            & 30.37                         \\
go-\mHC\ (s=1)         & $d=4$           & 0.46                            & 30.60                         \\
go-\mHC\ (s=2)         & $d=4$           & 1.34                            & 31.48                         \\ \hline
KromHC                 & $d=8$           & 0.84                            & 30.98                         \\
go-\mHC\ (s=1)         & $d=8$           & 2.98                            & 33.12                         \\
go-\mHC\ (s=2)         & $d=8$           & 10.06                           & 40.20                         \\
{\mHC-lite}     & $d=8$ & \textbf{1487.46}                & \textbf{1517.60}              \\ \hline
\end{tabular}
\end{table*}

    \section{Weighted-Average of Doubly-Stochastic Matrices}
    \label{app:alt_form}

    We discuss two alternative formulations one can use for parameterizing doubly-stochastic matrices. We add an additional $s$ learned parameters $\{\beta_i \vert i\in\{1, \cdots,s\}\}$ and instead weighted average a set of $s$ $d\times d$ orthostochastic matrices using the weights $\beta_i$.

    \subsection{Application to Orthostochastic\&Unistochastic Matrices}
    \label{app:alt_form_weight_avg_ortho}

    While we use the $d^2$ $s\times s$ blocks in our construction to build the class of matrices, this scales as $\cO(s^3d^3)$. We have found that an alternative formulation here where a transition matrix can be described as the sum of Orthostochastic (or Unistochastic) matrices to be effective at filling the Birkhoff polytope. This scales as $\cO(sd^3)$ instead of $\cO(s^3d^3)$ and fills the Birkhoff polytope in a different manner. Any Doubly Stochastic matrix, in the limit of multiple sums, can be represented as the sum of other doubly stochastic matrices. While the Karpelevi\v{c} region is filled entirely almost immediately (see figure~\ref{appfig:avg_method}), this prioritizes filling the facets of the Birkhoff polytope first (hence the Karpelevi\v{c} region being filled entirely). This method however leaves substantial gaps while filling the facets and therefore we opt to use the method in the main text. Further exploration of this would be interesting as a potentially lower-complexity way of learning the manifold.

    \begin{figure}
        \centering
        \includegraphics[width=\linewidth]{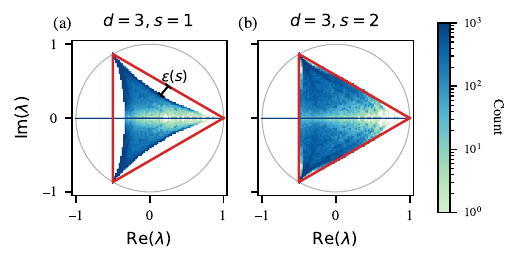}
        \caption{}
        \label{appfig:avg_method}
    \end{figure}
    
    \subsection{Application to KromHC}
    \label{app:alt_form_weight_avg_krom}

    KromHC introduces an inherent symmetry, we can attempt to break the symmetry by introducing weighted sums similar to section~\ref{app:alt_form_weight_avg_ortho} and summing the outcomes of indivdual KromHC forward passes. Emperically, we see that we need $n^2/\log(n)$ terms to achieve full expressivity, which takes us close to the limit of performance in go-\mHC.

\fi

\end{document}

%% file: table_llm_judge.tex
\begin{table}[ht]
\centering
\small
\caption{\textbf{Evaluation with GPT-4.1.} Results reflect the mean and standard error across $N=70$ samples, where each sample score is the consensus of 3 independent trials. \textbf{Bold} indicates the maximum value; \underline{underline} indicates scores within one standard error of the maximum. ICC(3,1) denotes intra-judge reliability across the 3 trials per sample.}
\label{tab:reg_70}
\resizebox{\linewidth}{!}{
\begin{tabular}{@{}lccc@{}}
\toprule
\textbf{Model} & \textbf{Grammar} & \textbf{Creativity} & \textbf{Consistency} \\ \midrule
Residual & $6.20 \pm 0.11$ & $\underline{5.87 \pm 0.09}$ & $4.83 \pm 0.09$ \\
\textit{m}HC-lite & $\underline{6.58 \pm 0.12}$ & $5.80 \pm 0.07$ & $\underline{5.13 \pm 0.11}$ \\
KromHC & $\underline{6.56 \pm 0.13}$ & $5.79 \pm 0.07$ & $5.07 \pm 0.11$ \\
go-\textit{m}HC ($s=1$) & $\underline{6.57 \pm 0.11}$ & $5.74 \pm 0.07$ & $\underline{5.12 \pm 0.09}$ \\
go-\textit{m}HC ($s=2$) & $\mathbf{6.63 \pm 0.12}$ & $\mathbf{5.88 \pm 0.06}$ & $\mathbf{5.22 \pm 0.11}$ \\ \midrule
\textbf{ICC(3,1)} & 0.890 & 0.671 & 0.825 \\ \bottomrule
\end{tabular}
}
\end{table}

%% file: tables_llm_judge_appendix.tex
\begin{table}[!ht]
\centering
\small
\caption{\textbf{Evaluation Performance: GPT-4.1 Mini.} Scores represent mean $\pm$ SEM ($N=70$). \textbf{Bold} indicates the maximum; \underline{underline} indicates scores within one standard error of the maximum.}
\label{tab:mini_70}
\resizebox{\linewidth}{!}{
\begin{tabular}{@{}lccc@{}}
\toprule
\textbf{Model} & \textbf{Grammar} & \textbf{Creativity} & \textbf{Consistency} \\ \midrule
Residual & $5.93 \pm 0.12$ & $\underline{5.84 \pm 0.06}$ & $4.58 \pm 0.09$ \\
\textit{m}HC-lite & $\mathbf{6.37 \pm 0.13}$ & $\mathbf{5.90 \pm 0.06}$ & $4.73 \pm 0.08$ \\
KromHC & $6.04 \pm 0.13$ & $5.82 \pm 0.06$ & $4.57 \pm 0.09$ \\
go-\textit{m}HC ($s=1$) & $\underline{6.26 \pm 0.13}$ & $\underline{5.88 \pm 0.07}$ & $4.64 \pm 0.08$ \\
go-\textit{m}HC ($s=2$) & $\underline{6.34 \pm 0.14}$ & $\underline{5.90 \pm 0.06}$ & $\mathbf{4.90 \pm 0.10}$ \\ \midrule
\textbf{ICC(3,1)} & 0.893 & 0.555 & 0.710 \\ \bottomrule
\end{tabular}
}
\end{table}

\begin{table}[!ht]
\centering
\small
\caption{\textbf{Evaluation Performance: GPT-4.1 Nano.} Scores represent mean $\pm$ SEM ($N=70$). \textbf{Bold} indicates the maximum; \underline{underline} indicates scores within one standard error of the maximum.}
\label{tab:nano_70}
\resizebox{\linewidth}{!}{
\begin{tabular}{@{}lccc@{}}
\toprule
\textbf{Model} & \textbf{Grammar} & \textbf{Creativity} & \textbf{Consistency} \\ \midrule
Residual & $6.20 \pm 0.10$ & $\underline{6.29 \pm 0.09}$ & $5.47 \pm 0.10$ \\
\textit{m}HC-lite & $\mathbf{6.47 \pm 0.10}$ & $\mathbf{6.33 \pm 0.07}$ & $\mathbf{5.82 \pm 0.09}$ \\
KromHC & $6.26 \pm 0.11$ & $6.25 \pm 0.08$ & $5.63 \pm 0.10$ \\
go-\textit{m}HC ($s=1$) & $\underline{6.45 \pm 0.10}$ & $\underline{6.30 \pm 0.10}$ & $5.70 \pm 0.09$ \\
go-\textit{m}HC ($s=2$) & $\underline{6.43 \pm 0.11}$ & $6.25 \pm 0.08$ & $\underline{5.80 \pm 0.11}$ \\ \midrule
\textbf{ICC(3,1)} & 0.834 & 0.605 & 0.685 \\ \bottomrule
\end{tabular}
}
\end{table}

\begin{table*}[ht]
\centering
\small
\caption{\textbf{Inter-Judge Alignment.} Pairwise comparison of rating consistency across $N=70$ samples. Spearman $\rho$ measures rank-order correlation, while Cross-Model ICC(3,1) assesses the absolute agreement between different model tiers acting as judges.}
\label{tab:pairwise_alignment}
\begin{tabular}{@{}llccc@{}}
\toprule
\textbf{Judge Pair} & \textbf{Metric} & \textbf{Grammar} & \textbf{Creativity} & \textbf{Consistency} \\ \midrule
\textbf{4.1 vs. 4.1 mini} & Spearman $\rho$ & 0.732 & 0.345 & 0.519 \\
                      & Cross-Model ICC & 0.817 & 0.535 & 0.639 \\ \addlinespace
\textbf{4.1 vs. 4.1 nano} & Spearman $\rho$ & 0.603 & 0.302 & 0.467 \\
                      & Cross-Model ICC & 0.748 & 0.405 & 0.567 \\ \addlinespace
\textbf{4.1 mini vs. 4.1 nano}& Spearman $\rho$ & 0.682 & 0.201 & 0.434 \\
                      & Cross-Model ICC & 0.786 & 0.318 & 0.393 \\ \bottomrule
\end{tabular}
\end{table*}